\documentclass{article}
\usepackage{float,graphicx,epstopdf}
\usepackage{bbm}  
\RequirePackage[reqno]{amsmath}
\usepackage{cases}

\usepackage[usenames, dvipsnames]{color}
\definecolor{darkblue}{rgb}{0.0, 0.0, 0.45}

\usepackage[colorlinks	= true,
raiselinks	= true,
linkcolor	= MidnightBlue, 
citecolor	= Mahogany,
urlcolor	= ForestGreen,
pdfauthor	= {},
pdftitle	= {},
pdfkeywords	= {},
pdfsubject	= {},
plainpages	= false]{hyperref}

\usepackage{dsfont,amssymb,amsmath,graphicx,enumitem, subfigure} 
\usepackage{amsfonts,dsfont,mathtools, mathrsfs,amsthm,fancyhdr}
\usepackage[amssymb, thickqspace]{SIunits}
\usepackage{enumitem}

\makeatletter
\def\munderbar#1{\underline{\sbox\tw@{$#1$}\dp\tw@\z@\box\tw@}}
\makeatother

\newtheorem{definition}{Definition}[section]
\newtheorem{theorem}[definition]{Theorem}
\newtheorem{lemma}[definition]{Lemma}

\newtheorem{corollary}[definition]{Corollary}
\newtheorem{proposition}[definition]{Proposition}

\newcommand{\be}{\begin{equation}}
\newcommand{\ee}{\end{equation}}
\newcommand{\bea}{\begin{equation*}\begin{aligned}}
\newcommand{\eea}{\end{aligned}\end{equation*}}
\newcommand{\ds}{\displaystyle}

\newcommand{\R}{\mathbb{R}}

\newcommand{\Max}{\max\limits_}
\newcommand{\Min}{\min\limits_}
\newcommand{\Sup}{\sup\limits_}
\newcommand{\Inf}{\inf\limits_}
\newcommand{\Tr}[1]{\Trace \big[ #1 \big]}

\newcommand{\wh}{\widehat}
\newcommand{\mc}{\mathcal}
\newcommand{\mbb}{\mathbb}

\newcommand{\cov}{\Sigma} 
\newcommand{\covsa}{\wh{\cov}}

\newcommand{\M}{\mc M}

\newcommand{\p}{\mbb P}
\newcommand{\PP}{\mbb P}

\newcommand{\QQ}{\mbb Q}

\DeclareMathOperator{\Trace}{Tr}

\DeclareMathOperator{\st}{s.t.}

\DeclareMathOperator{\KL}{KL}

\newcommand{\PSD}{\mathbb{S}_{+}} 
\newcommand{\PD}{\mathbb{S}_{++}} 
\newcommand{\Let}{\triangleq}
\newcommand{\opt}{^\star}
\newcommand{\eps}{\varepsilon}

\newcommand{\EE}{\mathds{E}}

\newcommand{\m}{\mu}
\newcommand{\msa}{\wh \mu}
\newcommand{\U}{\mc U}

\newcommand{\half}{\frac{1}{2}}
\newcommand{\dualvar}{\gamma}

\newcommand{\N}{\mc N}

\newcommand{\Pnom}{\wh \p}

\graphicspath{{graph/}}

\usepackage{microtype}
\usepackage{graphicx}
\usepackage{subfigure}
\usepackage{booktabs} 

\usepackage{hyperref}

\usepackage[accepted]{icml2020}

\icmltitlerunning{Robust Bayesian Classification Using an Optimistic Score Ratio}

\begin{document}

\twocolumn[
\icmltitle{Robust Bayesian Classification Using an Optimistic Score Ratio}



\icmlsetsymbol{equal}{*}

\begin{icmlauthorlist}
\icmlauthor{Viet Anh Nguyen}{FA}
\icmlauthor{Nian Si}{FA}
\icmlauthor{Jose Blanchet}{FA}
\end{icmlauthorlist}

\icmlaffiliation{FA}{Stanford University}

\icmlcorrespondingauthor{Viet Anh Nguyen}{viet-anh.nguyen@stanford.edu}

\icmlkeywords{binary classification, distributionally robust optimization}

\vskip 0.3in
]



\printAffiliationsAndNotice{}  

\begin{abstract}
    We build a Bayesian contextual classification model using an optimistic score ratio for robust binary classification when there is limited information on the class-conditional, or contextual, distribution. The optimistic score searches for the distribution that is most plausible to explain the observed outcomes in the testing sample among all distributions belonging to the contextual ambiguity set which is prescribed using a limited structural constraint on the mean vector and the covariance matrix of the underlying contextual distribution. We show that the Bayesian classifier using the optimistic score ratio is conceptually attractive, delivers solid statistical guarantees and is computationally tractable. We showcase the power of the proposed optimistic score ratio classifier on both synthetic and empirical data.
\end{abstract}

\section{Introduction}
\label{sect:intro}
We consider a binary classification setting in which we are provided with training samples from two classes but there is little structure within the classes, e.g., data with heterogeneous distributions except for means and covariance. The ultimate goal is to correctly classify an unlabeled test sample of a given feature. This supervised learning task is the cornerstone of modern machine learning, and its diverse applications are flourishing in promoting healthcare \cite{ref:naraei2016application,ref:tomar2013survey}, speeding up technological progresses \cite{ref:rippl2016limit,ref:zhu2014vehicle}, and improving societal values \cite{ref:bhagat2011node,ref:bodendorf2009detecting}. Confronting the unstructured nature of the problem, it is natural to exercise a Bayesian approach which employs subjective belief and available information, and then determine an optimal classifying decision that minimizes a certain loss function integrated under the posterior distribution. Although a consensus on the selection of the loss function can be easily reached, the choice of a class prior and a class-conditional distribution (i.e., the likelihood given the class), two compulsory inputs to the Bayesian machinery to devise the posterior, is more difficult to be agreed upon due to conflicting beliefs among involving parties and limited available data.

Robust Bayesian statistics, which explicitly aims to assemble a posterior inference model with multiple priors and/or multiple class-conditional distributions, emerges as a promising remedy to this longstanding problem. Existing research in this field mainly focuses on robust divergences in the general Bayesian inference framework.
 \citet{walker2013bayesian} identifies the behaviour of Bayesian updating in the context of model misspecification to show that standard Bayesian updating method learns a model that minimizes the Kullback-Leibler (KL) divergence to the true data generating model.  To achieve robustness in Bayesian inference, existing works often target robust divergences, including maximum mean discrepancy, R\'{e}nyi's alpha-divergences, Hellinger-based divergences, and density power divergence \cite{cherief2019mmd,knoblauch2019generalized,bissiri2016general,jewson2018principles,ghosh2016robust}. Learning the learning rate in the general Bayesian inference framework is also gaining more recent attention \cite{holmes2017assigning,knoblauch2019robust}. Besides, \citet{miller2019robust} use approximate Bayesian computation to obtain a 'coarsened' posterior to achieve robustness and \citet{grunwald2012safe} proposes a safe Bayesian method.
 
Despite being an active research field, alleviating the impact of the model uncertainty in the class-conditional distribution (i.e., the likelihood conditional on the class) using ideas from distributional robustness is left largely unexplored even though this uncertainty arises naturally for numerous reasons. Even if we assume a proper parametric family, the plug-in estimator still carries statistical error from finite sampling and rarely matches the true distribution. The uncertainty is amplified when one relaxes to the nonparametric setting where no hardwired likelihood specification remains valid, and we are not aware of any guidance on a reasonable choice of a likelihood in this case. The situation deteriorates further when the training data violates the independent or identically distributed assumptions, or when the test distribution differs from the training distribution as in the setting of covariate shift~\cite{ref:gretton2009covariate, ref:bicket2009discriminative, ref:moreno2012unifying}.

We endeavor in this paper to provide the precise mathematical model for binary classification with uncertain likelihood under the Bayesian decision analysis framework. Consider a binary classification setting where $Y \in \{0, 1\}$ represents the random class label and $X \in \R^d$ represents the random features. With a new observation $x$ to be classified, we consider the problem of finding an optimal action $a \in \{0, 1\}$,
\[
    a = \begin{cases}
    0 & \text{if classify $x$ in class 0}, \\
    1 & \text{if classify $x$ in class 1},
    \end{cases}
\]
to minimize the probability of  misclassification by solving the optimization problem
\[
    \Min{a \in \{0, 1\}}~a \PP( Y = 0 | X = x) + (1-a) \PP( Y = 1 | X = x),
\]
where $\PP(Y | X = x)$ denotes the posterior probability. If a class-proportion prior $\pi$ and the class-conditional (parametric) densities $f_0$ and $f_1$ are known, then this posterior probability can be calculated by using the Bayes' theorem~\citep[Theorem~1.31]{ref:schervish1995theory}.
Unluckily, we rarely have access to the true conditional densities in real life. 

To tackle this problem in the data-driven setting, for any class $c \in \{0, 1\}$, the decision maker first forms, to the best of its belief and on the availability of data, a nominal class-conditional distribution $\Pnom_c$. We assume now that the true class-conditional distribution belongs to an ambiguity set $\mbb B_{\rho_c}(\Pnom_c)$, defined as a ball, prescribed via an appropriate measure of dissimilarity, of radius $\rho_c \ge 0$ centered at the nominal distribution $\Pnom_c$ in the space of class-conditional probability measures. Besides, we allow to constrain the class-conditional distributions to lie in a subspace $\mc P$ of probability measures to facilitate the injection of optional parametric information, should the need arise.

To avoid any unnecessary measure theoretic complications, we position ourselves temporarily in the parametric setting and assume that we can generically write $\mbb B_{\rho_c}(\Pnom_c) \cap \mc P$ parametrically as
\[
    \mbb B_{\rho_c}(\Pnom_c) \cap \mc P = \left\{ f_c(\cdot | \theta_c) : \theta_c \in \Theta_c \right\} \quad \forall c \in \{0, 1\},
\]
where $\Theta_c$ are non-empty (sub)sets on the finite-dimensional parameter space $\Theta$, and $\Theta_c$ satisfy the additional regularity condition that the density evaluated at point $x$ is strictly positive, i.e., $f_c(x | \theta_c) > 0$ for all $\theta_c \in \Theta_c$. Notice that the parametric subspace of probability distributions $\mc P$ is now explicitly described through the set of admissible parameters $\Theta$. If we denote the prior proportions by $\pi_0 = \pi(Y=0) > 0$ and $\pi_1 = \pi(Y=1) > 0$, then the ambiguity set over the posterior distributions induced by the class-conditional ambiguity sets $\mbb B_0$ and $\mbb B_1$ can be written as
\begin{align*}
    \mc B\!
    =\!\left\{ \PP : \begin{array}{l}
        \exists f_0 \in \mbb B_{\rho_0}(\Pnom_0) \cap \mc P, f_1 \in \mbb B_{\rho_1}(\Pnom_1) \cap \mc P :\\
        \PP(Y = c | X = x) = \ds\frac{f_c(x) \pi_c}{ \ds\sum_{c' \in \{0, 1\}} f_{c'}(x) \pi_{c'} }~ \forall c\!\!
        \end{array}
    \right\},
\end{align*}
where the constraint in the set $\mc B$ links the class-conditional densities $f_c$ and the prior distribution of the class proportions $\pi$ to the posterior distribution. Facing with the uncertainty in the posterior distributions, it is reasonable to consider now the distributionally robust problem
\begin{align}
    \Min{a \in \{0, 1\}}\Sup{\PP \in \mc B}~a \PP( Y\!=\!0 | X\!=\!x) + (1-a) \PP( Y\!=\!1 | X\!=\!x), \label{eq:bayes-dro}
\end{align}
where the action $a$ is chosen so as to minimize the worst-case mis-classification probability over all posterior distribution $\PP \in \mc B$. The next proposition asserts that the optimal action $a\opt$ belongs to the class of ratio decision rule.
\begin{proposition}[Optimal action]
    \label{prop:optimal-action}
    The optimal action that minimizes the worst-case mis-classification probability~\eqref{eq:bayes-dro} has the form
    \[
        a\opt = \begin{cases}
            1 & \text{if } ~~\ds \frac{\sup_{f_1 \in \mbb B_{\rho_1}(\Pnom_1) \cap \mc P} f_1(x)}{\sup_{f_0 \in \mbb B_{\rho_0}(\Pnom_0) \cap \mc P} f_0(x)} \ge \tau(x), \\
            0 & \text{otherwise,}
        \end{cases}
    \]
    for some threshold $\tau > 0$ that is dependent on $x$.
\end{proposition}

Motivated by this insight from the parametric setting, we now promote the following 
classification decision rule
\begin{equation*}
\mc C (x) = \begin{cases}
1 &\text{if } \mc R(x) \ge \tau(x),\\
0 &\text{otherwise,}
\end{cases}
\end{equation*}
where $\mc R(x)$ is the ratio defined as
\[
    \mc R(x) \Let \ds \frac{\Sup{\QQ \in \mbb
B_{\rho_1}(\Pnom_1) \cap \mc P} \ell(x, \QQ)}{ \Sup{\QQ \in \mbb
B_{\rho_0}(\Pnom_0) \cap \mc P} \ell(x, \QQ)} \,,
\]
and $\tau(x) > 0$ is a positive threshold which is potentially dependent on the observation $x$. The \textit{score function} $\ell(x, \QQ)$ quantifies the plausibility of observing $x$ under the probability measure $\QQ$, and the value $\mc R(x)$ quantifies how plausible an observation $x$ can be generated by \textit{any} class-conditional probability distribution in $\mbb B_{\rho_1}(\Pnom_1) \cap \mc P$ relatively to \textit{any} distribution in $\mbb B_{\rho_0}(\Pnom_0) \cap \mc P$. Because both the numerator and the denominator search for the distribution in the respective ambiguity set that maximizes the score of observing $x$, $\mc R(x)$ is thus termed the \textit{ratio of optimistic scores}, and the classification decision $\mc C$ is hence called the \textit{optimistic score ratio classifier}.


The classifying decision $\mc C(x)$ necessitates the solution of two \textit{optimistic score evaluation problem}s of the form
\be \label{eq:prob}
\Sup{\QQ \in \mbb B_\rho(\Pnom) \cap \mc P}~\ell(x, \QQ),
\ee
where the dependence of the input parameters on the label $c \in \{0, 1\}$ has been omitted to avoid clutter. The performance of $\mc C$ depends critically on the specific choice of $\ell$ and $\mbb B_\rho(\Pnom)$. Typically, $\ell$ is subjectively tailored to the choice of a parametric or a nonparametric view on the conditional distribution, as we shall see later on in this paper. The construction of $\mbb B_\rho(\Pnom)$ is principally governed by choice of the dissimilarity measure that specifies the $\rho$-neighborhood of the nominal distribution $\Pnom$. Ideally, $\mbb B_\rho(\Pnom)$ should allow a coherent transition between the parametric and nonparametric setting via its interaction with the set $\mc P$. Furthermore, it should render problem~\eqref{eq:prob} computationally tractable with meaningful optimal value, and at the same time provide the flexibility to balance between exerting statistical guarantees and modelling domain adaptation. These stringent criteria precludes the utilization of popular dissimilarity measures in the emerging literature. Indeed, the likelihood problem using the $f$-divergence \cite{ref:bental2013robust, ref:namkoong2016stochstic} delivers unreasonable estimate in the nonparametric setting~\citep[Section~2]{ref:nguyen2019calculating}, the Wasserstein distance~\cite{ref:esfahani2018data,ref:kuhn2019wasserstein, ref:blanchet2019robust, ref:gao2016distributionally, ref:zhao2018data} typically renders the Gaussian parametric likelihood problem non-convex, and the maximum mean discrepancy~\cite{ref:iyer2014maximum, ref:staib2019distributionally} usually results in an infinite-dimensional optimization problem which is challenging to solve. This fact prompts us to explore an alternative construction of $\mbb B_\rho(\Pnom)$ that meets the criteria as mentioned above.

The contributions of this paper are summarized as follows.
\begin{itemize}[wide=0pt]
    \item We introduce a novel ambiguity set based on a divergence defined on the space of mean vector and covariance matrix. We show that this divergence manifests numerous favorable properties and evaluating the optimistic score is equivalent to solving a non-convex optimization problem. We prove the asymptotic statistical guarantee of the divergence, which directs an optimal calibration the size of the ambiguity set.
    \item We show that, despite its inherent non-convexity and hence intractability, the optimistic score evaluation problem can be efficiently solved in both nonparametric and parametric Gaussian settings. We reveal that the optimistic score ratio classifier generalizes the Mahalanobis distance classifier and the linear/quadratic discriminant analysis.
\end{itemize}

Because evaluating the plausibility of an observation $x$ is a fundamental problem in statistics, the results of this paper have far-reaching implications beyond the scope of the classification task. These include Bayesian inference using synthetic likelihood~\cite{ref:wood2010statistical, ref:price2018bayes}, approximate Bayesian computation~\cite{ref:csillery2010practice, ref:toni2009abc}, variational Bayes inference~\cite{ref:blei2017variational, ref:ong2018variational}, and composite hypothesis testing using likelihood ratio~\cite{ref:cox1961tests, ref:cox2013return}. These connections will be explored in future research.

All proofs are relegated to the appendix.

\textbf{Notations.} We let $\mc M$ be the set of probability measures supported on $\R^d$ with finite second moment. The set of (symmetric) positive definite
matrices is denoted by $\PD^d$. For any $\QQ \in \mc M$, $\m \in \R^d$
and $\cov \in \PD^d$, we use $\QQ \sim (\m, \cov)$ to express that $\QQ$ has
mean vector $\m$ and covariance matrix $\cov$. 
The $d$-dimensional identity matrix is denoted by $I_d$.
The space of Gaussian distributions is denoted by $\mc N$, and $\mc N(\m, \cov)$ denotes a Gaussian distribution with mean $\m$ and covariance matrix $\cov$. The trace and determinant operator are denoted by $\Tr{A}$ and $\det(A)$, respectively.

\section{Moment-based Divergence Ambiguity Set}
\label{sect:ambiguity}

We specifically study the construction of the ambiguity set using the following divergence on the space of moments.
\begin{definition}[Moment-based divergence] \label{def:divergence}
    For any vectors $\m_1$, $\m_2 \in \R^d$ and matrices $\cov_1$, $\cov_2 \in \PD^d$, the divergence from the tuple $(\m_1, \cov_1)$ to the tuple $(\m_2, \cov_2)$ amounts to
    \begin{align*}
        &\mathds D \big( (\m_1, \cov_1) \parallel (\m_2, \cov_2) \big) \Let (\m_2 - \m_1)^\top\cov_2^{-1} (\m_2 - \m_1) \\
        &\qquad +\Tr{\cov_1 \cov_2^{-1}} - \log\det (\cov_1 \cov_2^{-1}) - d.
    \end{align*}
\end{definition}
To avoid any confusion, it is worthy to note that contrary to the usual utilization of the term `divergence' to specify a dissimilarity measure on the probability space, in this paper, the divergence is defined on the finite-dimensional space of mean vectors and covariance matrices. 

It is straightforward to show that $\mathds D$ is a divergence on $\R^d \times \PD^d$ by noticing that $\mathds D$ is a sum of the log-determinant divergence~\cite{ref:chebbi2012means} from $\cov_1$ to $\cov_2$ and a non-negative Mahalanobis distance between $\m_1$ and $\m_2$ weighed by $\cov_2$. As a consequence, $\mathds D$ is non-negative, and perishes to 0 if and only if $\cov_1 = \cov_2$ and $\m_1 = \m_2$. With this property, $\mathds D$ is an attractive candidate for the divergence on the joint space of mean vector and covariance matrix of $d$-dimensional random vectors. One can additionally verify that $\mathds D$ is  affine-invariant in the following sense. Let $\xi$ be a $d$-dimensional random vector and $\zeta$ be the affine-transformation of $\xi$, that is, $\zeta = A\xi+b$ for an invertible matrix $A$ and a vector $b$ of matching dimensions, then the value of the divergence $\mathds D$ is preserved between the space of moments of $\xi$ and $\zeta$. In fact, if $\xi$ is a random vector with mean vector $\m_j \in \R^d$ and covariance matrix $\cov_j \in \PD^d$, then $\zeta$ has mean $A\m_j+b$ and covariance matrix $A\cov_j A^\top$ for $j \in \{1, 2\}$, and we have
\begin{align}
    &\mathds D\big((\m_1, \cov_1) \parallel (\m_2, \cov_2) \big) \label{eq:linear_transform} \\
     &= \mathds D\big((A\m_1+b, A\cov_1 A^\top) \parallel (A\m_2+b, A\cov_2 A^\top) \big).  \notag
\end{align}
A direct consequence is that $\mathds D$ is also scale-invariant.
Furthermore, the divergence $\mathds D$ is closely related to the KL divergence\footnote{If $\QQ_1$ is absolutely continuous with respect to $\QQ_2$, then the Kullback-Leibler divergence from $\QQ_1$ to $\QQ_2$ amounts to $\KL(\QQ_1 \parallel \QQ_2) \Let \EE_{\QQ_1} [\log \mathrm{d}\QQ_1 /\mathrm{d} \QQ_2]$,
where $\mathrm{d}\QQ_1 /\mathrm{d} \QQ_2$ is the Radon-Nikodym derivative of $\QQ_1$ with respect to $\QQ_2$.}, or the relative entropy, between two non-degenerate Gaussian distributions as
\[
    \mathds D\big((\m_1, \cov_1)\!\parallel\!(\m_2, \cov_2) \big)\!=\!2\KL\big( \mc N(\m_1, \cov_1)\!\parallel\!\mc N(\m_2, \cov_2) \big).
\]
However, we emphasize that $\mathds D$ is not symmetric, and in general $\mathds D\big((\m_1, \cov_1) \parallel (\m_2, \cov_2) \big) \neq \mathds D\big((\m_2, \cov_2) \parallel (\m_1, \cov_1) \big)$. Hence, $\mathds D$ is not a distance on $\R^d \times \PD^d$.

For any vector $\msa \in \R^d$, invertible matrix $\covsa \in \PD^d$ and radius $\rho \in \R_+$, we define the uncertainty set $\mc U_\rho(\msa, \covsa)$ over the mean vector and covariance matrix space as
\be \label{eq:U-def}
\begin{aligned}
&\mc U_{\rho}({\msa}, \covsa) \Let \\
&\quad \{ (\m, \cov) \in \R^d \times \PD^d: \mathds D \big( (\msa, \covsa)
\parallel (\m, \cov) \big) \leq \rho \}.
\end{aligned}
\ee
By definition, $\mc U_\rho(\msa, \covsa)$ includes all tuples $(\m, \cov)$ which is of a divergence not bigger than $\rho$ from the tuple $(\msa, \covsa)$. Because $\mathds D$ is not symmetric, it is important to note that $\mc U_\rho(\msa, \covsa)$ is defined with the tuple $(\msa, \covsa)$ being the first argument of the divergence $\mathds D$, and this uncertainty set can be written in a more expressive form as
\begin{align*}
&\mc U_{\rho}({\msa}, \covsa) = \\
&\left\{ \begin{array}{l}
\!\!(\m, \cov) \in \R^d \times \PD^d: \\
\!\!(\m - \msa)^\top\cov^{-1} (\m - \msa) +\Tr{\covsa \cov^{-1}} + \log\det \cov \leq \overline \rho \!\!
\end{array}
\right\}
\end{align*}
for a scalar $\overline \rho \Let \rho + d + \log\det\covsa$. Moreover, one can assert that $\mc U_\rho(\msa, \covsa)$ is non-convex due to the log-determinant term, and this non-convexity cannot be eliminated using the reparametrization to the space of inverse covariance matrices (or equivalently called the precision matrices).

Equipped with $\mc U_\rho(\msa, \covsa)$, the ambiguity set $\mbb B_\rho(\Pnom)$ is systematically constructed as follows. If the nominal distribution $\Pnom$ admits a nominal mean vector $\msa$ and a nominal nondegenerate covariance matrix $\covsa$, then $\mbb B_\rho(\Pnom)$ is a ball that contains all probability measures whose mean vector and covariance matrix are contained in $\mc U_\rho(\msa, \covsa)$, that is,
\be \label{eq:B-def}
\begin{aligned}
\mbb B_\rho(\Pnom)\!\Let\! \{ \QQ \in \M\!:\!\QQ \sim (\m, \cov), (\m, \cov)
\in \mc U_\rho(\msa, \covsa) \}.
\end{aligned}
\ee
The set $\mbb B_\rho(\Pnom)$, by construction, differentiates only through the information about the first two moments: if a distribution $\QQ$ belongs to $\mbb B_\rho(\Pnom)$, then \textit{any} distribution $\QQ'$ with the same mean vector and covariance matrix with $\QQ$ also belongs to $\mbb B_\rho(\Pnom)$. Further, $\mbb B_\rho(\Pnom)$ embraces all types of probability distributions, including discrete, continuous and even mixed continuous/discrete distributions.

We now delineate a principled approach to solve the optimistic score evaluation problem~\eqref{eq:prob} for a generic score function $\ell: \R^d \times \M \to \R$.
We denote by $\mc M(\m, \cov)$ the Chebyshev ambiguity set that contains
all probability measures with \textit{fixed} mean vector $\m \in \R^d$ and \textit{fixed} covariance
matrix $\cov \in \PD^d$, that is,
\begin{equation*}
\mc M(\m, \cov) \Let\left\{ \QQ \in \mc M: \QQ \sim (\m, \cov) \right\}.
\end{equation*}
The moment-based divergence ambiguity set $\mbb B_\rho(\Pnom)$ then admits an equivalent representation
\begin{equation*}
\mbb B_\rho(\Pnom) = \bigcup_{(\m, \cov) \in \mc U_\rho(\msa, \covsa)} \mc M(\m, \cov),
\end{equation*}
which is an infinite union of Chebyshev ambiguity sets, where the union operator is taken over all tuples of mean vector-covariance matrix belonging to $\mc U_\rho(\msa, \covsa)$. Leveraging on this representation, problem~\eqref{eq:prob} can now be decomposed as a two-layer optimization problem
\be \label{eq:two-layer} \Sup{\QQ \in \mbb B_\rho(\Pnom)}%
~\ell(x, \QQ) = \Sup{(\m, \cov) \in \mc U_\rho(\msa, \covsa)}~%
\Sup{\QQ \in
\mc M(\m, \cov) \cap \mc P}~\ell(x, \QQ).
\ee
The inner subproblem of~\eqref{eq:two-layer} is a distributionally robust optimization problem with a Chebyshev second moment ambiguity set, hence there is a strong potential to exploit existent results from the literature, see \citet{ref:delage2010distributionally} and \citet{ref:wiesemann2014distributionally}, to reformulate this inner problem into a finite dimensional convex optimization problem. Unfortunately, the outer subproblem of~\eqref{eq:two-layer} is a robust optimization problem over a non-convex uncertainty set $\mc U_\rho(\msa, \covsa)$, thus the two-layer decomposition problem~\eqref{eq:two-layer} remains computationally intractable in general. As a direct consequence, solving the optimistic score evaluation problem requires an intricate adaptation of non-convex optimization techniques applied on a case-by-case basis. Two exemplary settings in which problem~\eqref{eq:two-layer} can be efficiently solved will be depicted subsequently in Sections~\ref{sect:nonparam} and~\ref{sect:param}.

We complete this section by providing the asymptotic statistical guarantees of the divergence $\mathds D$, which serves as a potential guideline for the construction of the ambiguity set $\mbb B_\rho(\Pnom)$ and the tuning of the radius parameter $\rho$.

\begin{theorem}[Asymptotic guarantee of $\mathds D$] \label{thm:asymptotic}
    Suppose that a $d$-dimensional random vector $\xi$ has mean vector $m \in \R^d$, covariance matrix $S \in \PD^d$ and admits finite fourth moment under a probability measure $\PP$. Let $\wh \xi_t \in \R^d$, $t=1, \ldots, n$ be independent and identically distributed samples of $\xi$ from $\PP$. Denote by $\msa_n \in \R^d$ and $\covsa_n \in \PSD^d$ the sample mean vector and sample covariance matrix defined as
    \be \label{eq:sample-average}
        \msa_n = \frac{1}{n} \sum_{t=1}^n \wh \xi_t, \quad \covsa_n = \frac{1}{n} \sum_{t=1}^n (\wh \xi_t - \msa_n)(\wh \xi_t - \msa_n)^\top.
    \ee
    Let $\eta = S^{-\half}(\xi - m)$ be the isotropic transformation of the random vector $\xi$, let $H$ be a $d$-dimensional Gaussian random vector with mean vector 0 and covariance matrix $I_d$, and let $Z$ be a $d$-by-$d$ random symmetric matrix with the upper triangle component $Z_{jk}$ ($j\leq k$) following a Gaussian distribution with mean 0 and the covariance coefficient between $Z_{jk}$ and $Z_{j^{\prime }k^{\prime }}$ is
    \[
        \mathrm{cov}(Z_{jk},Z_{j' k'})=\EE_{\PP}[\eta_j \eta_k \eta_{j'} \eta_{k'}] -\EE_{\PP}[\eta_j \eta_k]\,\EE_{\PP}[\eta_{j'}\eta_{k'}].
    \]
    Furthermore, $H$ and $Z$ are jointly Gaussian distributed with the covariance between $H_i$ and $Z_{jk}$ as
\[
\mathrm{cov}(H_i,Z_{jk}) = \EE_{\PP}[\eta_i \eta_j \eta_k].
\]
As $n \uparrow \infty$, we have
    \begin{align*}
        &n \times \mathds D \big( (\msa_n, \covsa_n) \parallel (m, S) \big) \\
        & \qquad  \longrightarrow H^{\top} H + \frac{1}{2} \Tr{Z^2} \quad \text{in distribution.}
    \end{align*}
\end{theorem}

We were not able to locate Theorem~\ref{thm:asymptotic} in the existing literature. Interestingly, Theorem~\ref{thm:asymptotic} also sheds light upon the asymptotic behavior of the KL divergence from an empirical Gaussian distribution $\mc N(\msa_n, \covsa_n)$ to the data-generating Gaussian distribution $\mc N(m, S)$.
\begin{corollary}[Asymptotic guarantee of $\mathds D$ -- Gaussian distributions] \label{cor:gaussian}
    Suppose that $\wh \xi_t \in \R^d$, $t=1, \ldots, n$ are independent and identically distributed samples of $\xi$ from $\PP = \mc N(m,S)$ for some $m \in \R^d$ and $S \in \PD^d$. Let $\msa_n \in \R^d$ and $\covsa_n \in \PSD^d$ be the sample mean vector and covariance matrix defined as in~\eqref{eq:sample-average}. As $n \uparrow \infty$, we have
    \begin{align*}
        & n \times \KL \big( \mc N (\msa_n, \covsa_n) \parallel \mc N(m, S) \big) \\
        &\qquad  \longrightarrow \frac{1}{2}\chi^2\left(d(d+3)/2\right)  \quad \text{in distribution,}
    \end{align*}
    where $\chi^2\left(d(d+3)/2\right)$ is a chi-square distribution with  $d(d+3)/2$ degrees of freedom.
\end{corollary}

If we use independent and identically distributed (i.i.d.) samples to estimate the nominal mean vector and covariance matrix of $\Pnom$, then the radius $\rho$ should be asymptotically scaled at the rate $n^{-1}$ as the sample size $n$ increases. Indeed, Theorem~\ref{thm:asymptotic} and Corollary~\ref{cor:gaussian} suggest that $n^{-1}$ is the optimal asymptotic rate which ensures that the true but unknown mean vector and covariance matrix of the data-generating distribution fall into the set $\mc U_\rho(\msa, \covsa)$ with high probability. While the limiting distribution under the Gaussian setting is a typical chi-square distribution, the general limiting distribution $H^\top H + \Tr{Z^2}/2$ in Theorem~\ref{thm:asymptotic} does not have any analytical form. This limiting distribution can be numerically approximated, for example, via Monte Carlo simulations. If the i.i.d.~assumption of the training samples is violated or if we expect a covariate shift at test time, then the radius $\rho$ reflects the modeler's belief regarding the moment mismatch measured using the divergence $\mathds D$ and in this case, the radius $\rho$ should be considered as an exogenous input to the problem.

For illustrative purpose, we fix dimension $d = 20$ and consider the random vector $\xi = C
\zeta + m$, where entries of $\zeta$ are mutually independent and the $i$-th entry follows a normalized chi-square distribution, i.e., $\zeta_i \sim
(\chi^2(1) - 1)/\sqrt{2} $. Then the covariance matrix of $\xi$ is $S=CC^{\top}$. Notice that by the identity \eqref{eq:linear_transform},   $\mathds D\big( (\msa_n, \covsa%
_n) \parallel (m, S) \big)$ is invariant of the choice of $C$ and $m$. We generate 10,000 datasets, each contains $n$ i.i.d.~samples of $\xi$ and calculate for each dataset the empirical values of $n\times \mathds D\big( (\msa_n, \covsa%
_n) \parallel (m, S) \big)$. We plot in Figure \ref{plot:clt} the empirical distribution of $n\times \mathds D\big( (\msa_n, \covsa%
_n) \parallel (m, S) \big)$ using 10,000 datasets versus the limiting distribution of $H^{\top} H +
\Tr{Z^2}/2$ for different values of $n$. One can observe that for a small sample size ($n < 100$), there is a perceivable
difference between the finite sample distribution and the limiting
distribution, but as $n$ becomes larger ($n > 100$), this mismatch is significantly reduced.

\begin{figure}[ptbh]
\centering
\subfigure[$n = 30$]{
 \includegraphics[width=0.45\columnwidth]{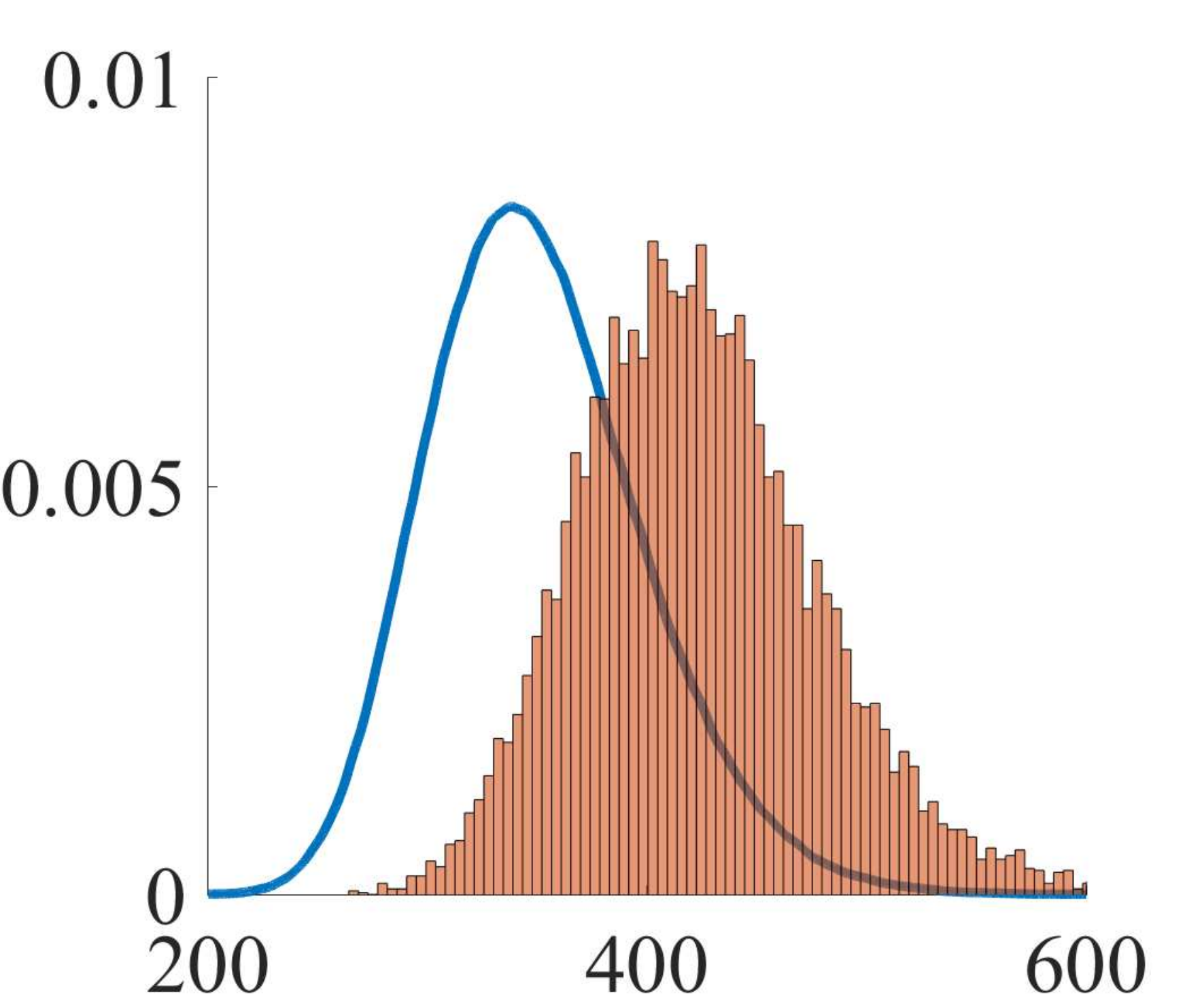}}
\subfigure[$n = 100$]{
 \includegraphics[width=0.45\columnwidth]{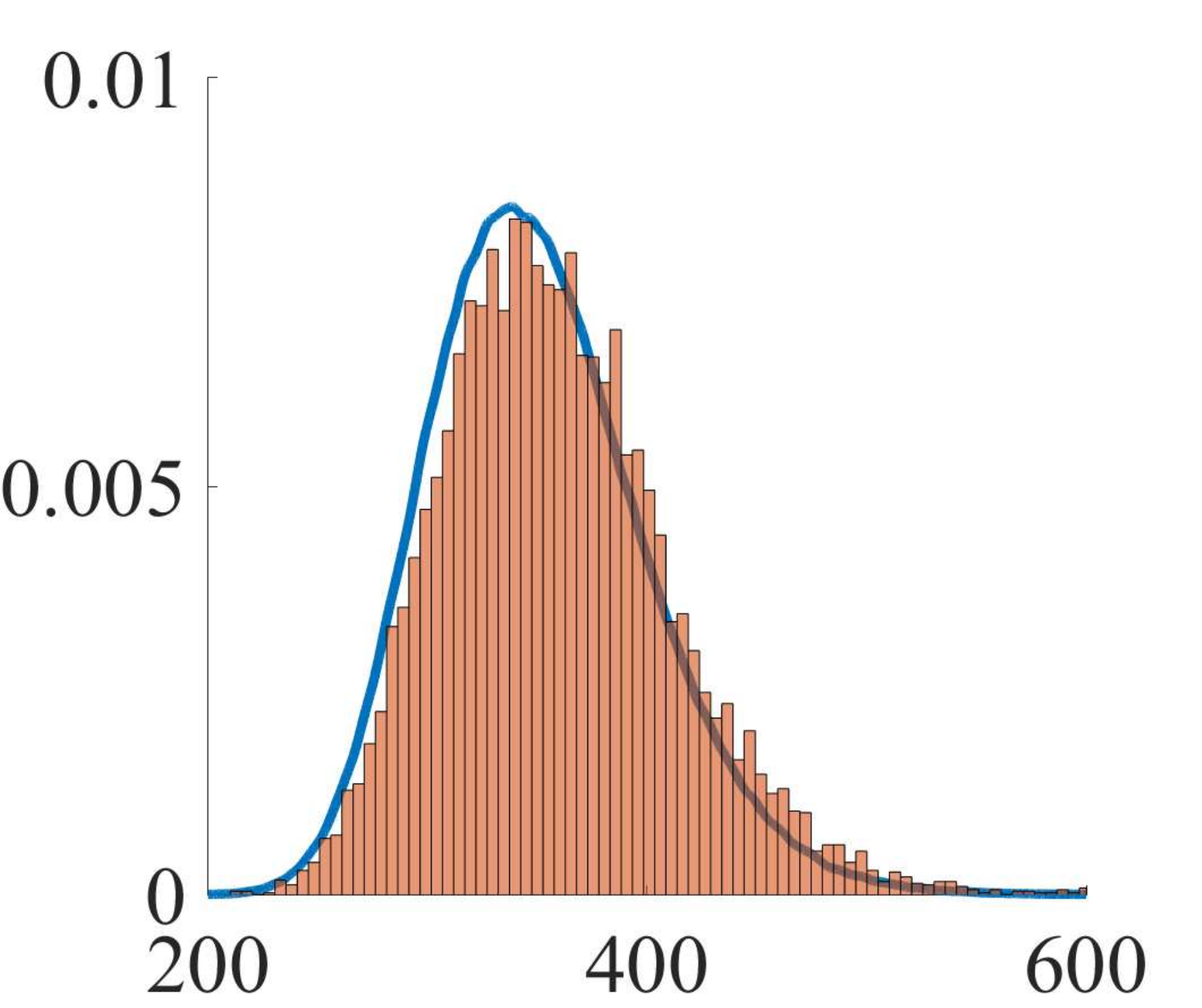}} \\
\subfigure[$n = 300$]{
 \includegraphics[width=0.45\columnwidth]{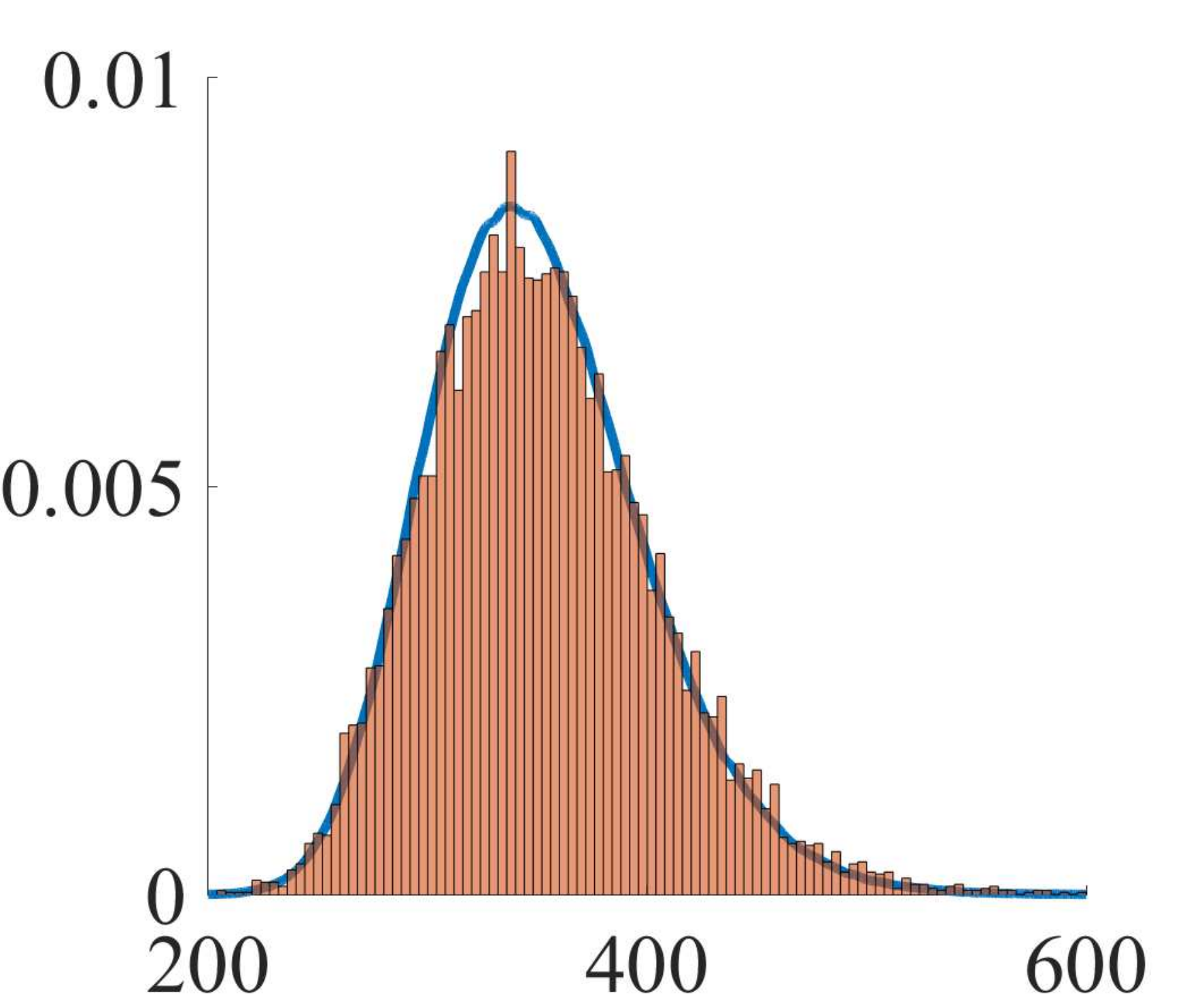}}
\subfigure[$ n = 1000$]{
 \includegraphics[width=0.45\columnwidth]{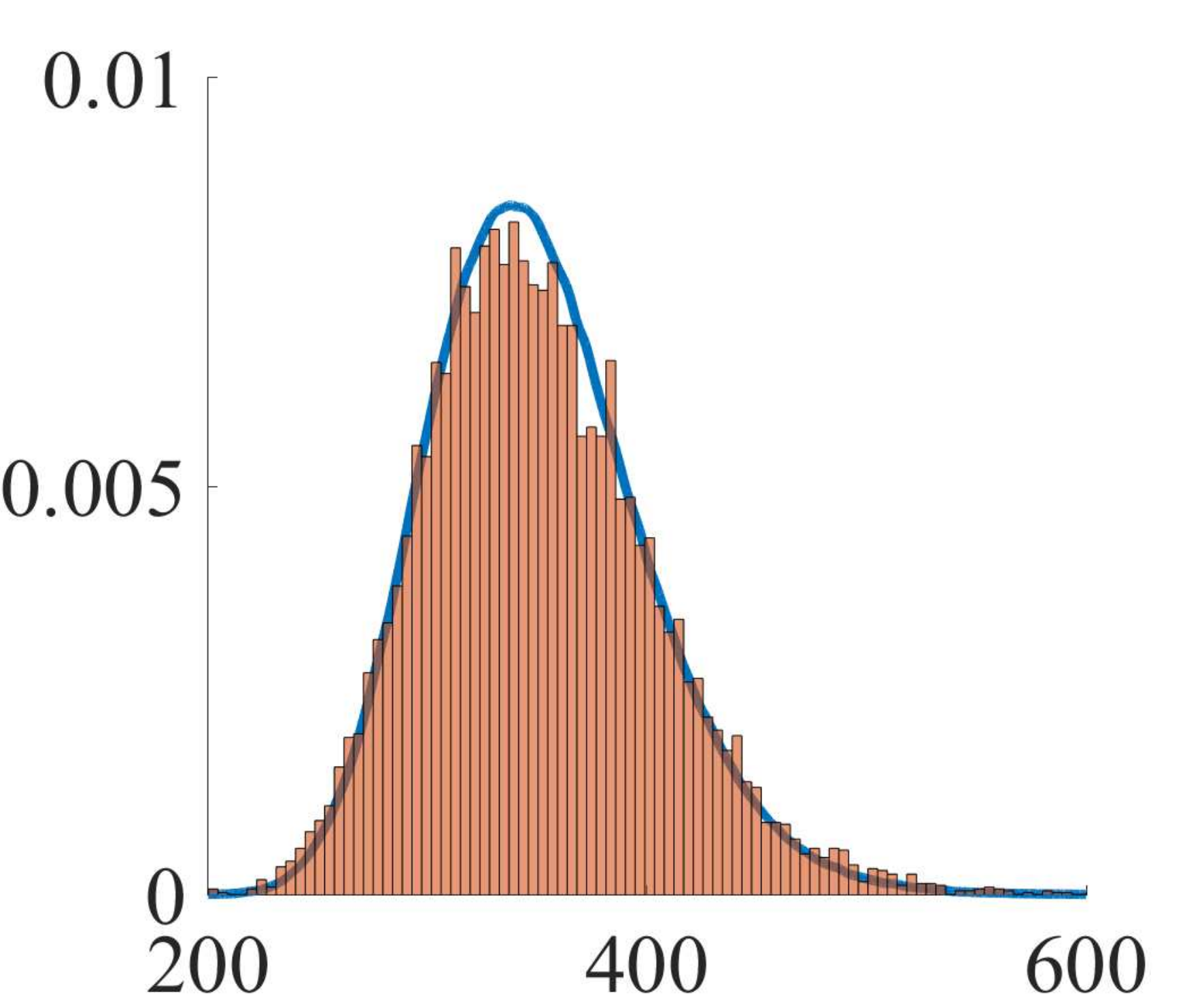}}
\caption{Empirical distribution of $n \times \mathds D\big( (\msa_n, \covsa_n) \parallel (m, S) \big)$ collected from 10,000 datasets (orange histogram) versus the limiting distribution $H^{\top} H + \frac{1}{2} \Tr{Z^2}$ obtained by Monte Carlo simulations (blue curve) for different sample sizes $n$.}
\label{plot:clt}
\end{figure}

\section{Optimistic Nonparametric Score}
\label{sect:nonparam}

We consider in this section the nonparametric setting in which no prior assumption on the class-conditional distribution is imposed. A major difficulty in this nonparametric setting is the elicitation of a reasonable score function $\ell$ that can coherently encapsulate the plausibility of observing $x$ over the whole spectrum of admissible $\QQ$, including continuous, discrete and mixed continuous/discrete distributions, while at the same time being amenable for optimization purposes. Taking this fact into consideration, we thus posit to choose the score function of the form
\begin{equation*}
\ell(x, \QQ) \equiv \QQ(\{x\}),
\end{equation*}
which is the probability value of the singleton, measurable set $\{x\}$ under the measure $\QQ$. If $\QQ$ is a continuous distribution, then apparently $\QQ(\{x\})$ is zero, hence this score function is admittedly not perfect. Nevertheless, it serves as a sensible proxy in the nonparametric setting and delivers competitive performance in machine learning tasks~\cite{ref:nguyen2019optimistic}. It is reasonable to set  $\mc P \equiv \mc M$ in the nonparametric setting, and with this choice of $\ell$, the optimistic nonparametric score evaluation problem becomes
\[
    \Sup{\QQ \in \mbb B_\rho(\Pnom)} ~\QQ(\{x\}),
\]
which is inherently challenging because it is an
infinite-dimensional optimization problem. The next theorem asserts that solving the nonparametric optimistic likelihood optimization problem is equivalent to solving a univariate convex optimization problem.

\begin{theorem}[Optimistic nonparametric probability] \label{thm:nonparam}
    \begin{subequations}
    Suppose that $\Pnom\sim(\msa, \covsa)$ for some $\msa \in \R^d$ and $\covsa \in \PD^d$. For any $\rho \in \R_+$, we have
    \begin{align}
    &\Sup{\QQ \in \mbb B_\rho(\Pnom)} ~\QQ(\{x\})  \notag \\
    =& \Max{(\m, \cov) \in \mc U_\rho(\msa, \covsa)}~[1+ (\m - x)^\top \cov^{-1} (\m - x)]^{-1} \label{eq:non-param:refor0}\\
	=& [1+ (\m\opt - x)^\top (\cov\opt)^{-1} (\m\opt - x)]^{-1}, \label{eq:non-param:refor00}
    \end{align}
where $(\m\opt, \cov\opt) \in \R^d \times \PD^d$ satisfies
\be \label{eq:non-param:refor1}
    \begin{aligned}
    \m\opt &= \frac{x + \dualvar\opt \msa}{1 + \dualvar\opt}, \\
    \cov\opt &= \covsa +  \frac{1}{(1+\dualvar\opt)} (x - \msa)(x - \msa)^\top,
    \end{aligned}
\ee
and $\dualvar\opt \in \R_+$ solves the univariate convex optimization problem
\be \label{eq:non-param:refor2}
\Min{\dualvar \ge 0 }~\dualvar \rho - \dualvar\log \Big( 1 + \frac{(x - \msa)^\top \covsa^{-1} (x - \msa)}{1+\dualvar}   \Big).
\ee
\end{subequations}
\end{theorem}

Because the feasible set $\mbb B_\rho(\Pnom)$ is not weakly compact,
the existence of an optimal measure that solves the optimistic likelihood problem on the left-hand side of~\eqref{eq:non-param:refor0} is not
trivial. However, equation~\eqref{eq:non-param:refor0} asserts that this optimal measure exists, and it can be constructed by solving a non-convex optimization problem over the mean vector-covariance matrix tuple $(\m, \cov)$. Notice that~\eqref{eq:non-param:refor0} is a non-convex optimization problem
because $\mc U_\rho(%
\msa, \covsa)$ is a non-convex set. Surprisingly, one can show that the optimizer of~\eqref{eq:non-param:refor0} can be found semi-analytically: the maximizer~$(\m\opt, \cov\opt)$ depends only on a single scalar $\dualvar\opt$ through~\eqref{eq:non-param:refor1}, where $\dualvar\opt$ solves the univariate optimization problem~\eqref{eq:non-param:refor2}. Because problem~\eqref{eq:non-param:refor2} is convex, $\dualvar\opt$ can be efficiently found using a bisection algorithm or using a Newton-Raphson method, and we expose in Appendix~\ref{sect:gradient} the first- and second-order derivative of the objective function of~\eqref{eq:non-param:refor2}.

A nonparametric classifier $\mc C_{\text{nonparam}}$ can be formed by utilizing the optimistic nonparametric score ratio
\be \label{eq:nonparam-ratio}
\mc R_{\text{nonparam}} (x) \Let \frac{\Sup{\QQ \in \mbb
B_{\rho_1}(\Pnom_1)} \QQ(\{x\}) }{ \Sup{\QQ \in \mbb
B_{\rho_0}(\Pnom_0)} \QQ(\{x\})},
\ee
where each nominal class-conditional distribution $\Pnom_c$ has mean vector ${\msa_c \in \R^d}$ and covariance matrix $\covsa_c \in \PD^d$, and each ambiguity set $\mbb B_{\rho_c}(\Pnom_c)$ is defined as in~\eqref{eq:B-def}. The results of Theorem~\ref{thm:nonparam} can be used to compute the numerator and denominator of~\eqref{eq:nonparam-ratio}, thus the classification decision~$\mc C_{\text{nonparam}}(x)$ can be efficiently evaluated. In particular, by substituting the expression~\eqref{eq:non-param:refor0} into~\eqref{eq:nonparam-ratio}, we also find
\begin{align*}
        &\mc R_{\text{nonparam}} (x) \\
        =&\frac{\Max{(\m, \cov) \in \mc U_{\rho_1}(\msa_1, \covsa_1)}~[1+ (\m - x)^\top \cov^{-1} (\m - x)]^{-1}}{\Max{(\m, \cov) \in \mc U_{\rho_0}(\msa_0, \covsa_0)}~[1+ (\m - x)^\top \cov^{-1} (\m - x)]^{-1}} \\
        =&\frac{1+ \Min{(\m, \cov) \in \mc U_{\rho_0}(\msa_0, \covsa_0)}(\m - x)^\top \cov^{-1} (\m - x)}{1+ \Min{(\m, \cov) \in \mc U_{\rho_1}(\msa_1, \covsa_1)}(\m - x)^\top \cov^{-1} (\m - x)}.
        \end{align*}
Suppose that $\rho_0 = \rho_1 = 0$ and $\tau(x) = 1$, then the nonparametric classifier assigns~$\mc C_{\text{nonparam}} (x) = 1$ whenever
\[ (\msa_1 - x)^\top \covsa_1^{-1} (\msa_1 - x) \le (\msa_0 - x)^\top \covsa_0^{-1} (\msa_0 - x),
\]
and $\mc C_{\text{nonparam}} (x) = 0$ otherwise. In this case, the classifier coincides with the class-specific Mahalanobis distance classifier (MDC) where $\covsa_0$ and $\covsa_1$ denote the intra-class nominal covariance matrices. If in addition the nominal covariance matrices are homogeneous, that is, $\covsa_0 = \covsa_1$, then this classifier coincides with the Linear Discriminant Analysis (LDA)~\citep[Section~4.2.2]{ref:murphy2012machine}. The Bayesian version of LDA can be equivalently obtained from $\mc C_{\text{nonparam}}$ by setting a proper value of $\tau(x)$. This important observation reveals an intimate link between our proposed classifier $\mc C_{\text{nonparam}}$ using the optimistic nonparametric score ratio and the popular classifiers MDC and LDA. On the one hand, $\mc C_{\text{nonparam}}$ can now be regarded as a generalization of MDC and LDA, which takes into account the statistical imprecision of the estimated moments and/or the potential shift in the moment statistics in test data versus training data distributions. On the other hand, both MDC and LDA now admit a nonparametric, generative interpretation in which the class-conditional distribution is chosen in the set of all distributions with the same first- and second-moments as the nominal class-conditional measure $\Pnom_c$. This novel interpretation goes beyond the classical Gaussian model, and it potentially explains the versatile performance of MDC and LDA when the conditional distribution are not normally distributed as empirically observed in \cite{ref:lee2018simple}.

\section{Optimistic Gaussian Score}
\label{sect:param}

We now consider the optimistic score evaluation problem under a parametric setting. For simplicity, we assume that the true class-conditional distributions of the feature belong to the family of Gaussian distributions. Thus, a natural choice of the score value $\ell(x, \QQ)$ in this case is the Gaussian likelihood of an observation $x$ when $\QQ$ is a Gaussian distribution with mean $\m$ and covariance matrix $\cov$, that is,
\[
    \ell(x, \QQ) = \frac{1}{\sqrt{(2\pi)^d \det \cov}} \exp\Big(-\frac{ (x - \m)^\top\cov^{-1} (x-\m)}{2} \Big).
\]
It is also suitable to set $\mc P$ in problem~\eqref{eq:prob} to the (sub)space of Gaussian distributions $\mc N$ and consider the following optimistic Gaussian score evaluation problem
\be \label{eq:Gauss-likelihood}
\Sup{\QQ \in \mathbb{B}_\rho(\Pnom) \cap \N}~ \ell(x, \QQ).
\ee
One can verify that the maximizer of problem~\eqref{eq:Gauss-likelihood} coincides with the maximizer of
\be \notag
\Sup{\QQ \in \mathbb{B}_\rho(\Pnom) \cap \N}~ \mc L(x, \QQ),
\ee
where $\mc L$ is the translated Gaussian \textit{log}-likelihood defined as
\begin{align*}
\mc L(x, \QQ) &=2\left( \log\big(\ell(x, \QQ)\big) + \frac{d}{2} \log(2\pi) \right)\\
&=-(\m - x)^\top \cov^{-1} (\m - x) - \log\det \cov.
\end{align*}
Theorem~\ref{thm:Gauss} is a counterpart to the optimistic nonparametric likelihood presented in Theorem~\ref{thm:nonparam}.

\begin{theorem}[Optimistic Gaussian log-likelihood] \label{thm:Gauss}
    Suppose that $\Pnom\sim\N(\msa, \covsa)$ for some $\msa \in \R^d$ and $\covsa \in \PD^d$. For any $\rho \in \R_+$, we have
\begin{subequations}
    \begin{align}
    &\Sup{\QQ \in \mbb B_\rho(\Pnom) \cap \N} ~ \mc L(x, \QQ) \notag \\
    =& \Max{(\m, \cov) \in \mc U_\rho(\msa, \covsa)}~-(\m - x)^\top \cov^{-1} (\m - x) -\log\det \cov \label{eq:gauss:refor0} \\
    =& -(\m\opt - x)^\top (\cov\opt)^{-1} (\m\opt - x) -\log\det \cov\opt, \label{eq:gauss:refor00}
    \end{align}
where $(\m\opt, \cov\opt) \in \R^d \times \PD^d$ satisfies
\be \label{eq:gauss:refor1}
    \begin{aligned}
    \m\opt &= \frac{x + \dualvar\opt \msa}{1 + \dualvar\opt},\\
    \cov\opt &=  \frac{\dualvar\opt}{1+\dualvar\opt} \covsa +  \frac{\dualvar\opt}{(1+\dualvar\opt)^2} (x - \msa)(x - \msa)^\top,
    \end{aligned}
\ee
and $\dualvar\opt \in \R_+$ solves the univariate convex optimization problem
\be \label{eq:gauss:refor2}
\begin{aligned}
\Min{\dualvar \ge 0 }&\Big\{ \dualvar\rho + {d}(\dualvar + 1) \log\Big( 1 + \frac{1}{\dualvar} \Big) \\
& \!\!\! - (1+\dualvar) \log \Big(1 + \frac{ (x - \msa)^\top \covsa ^{-1}(x - \msa)}{ (1+\dualvar)}\Big) \Big\}.
\end{aligned}
\ee
\end{subequations}
\end{theorem}

Notice that we impose the condition $\Pnom\sim \mc N(\msa, \covsa)$ in Theorem~\ref{thm:Gauss} to conform with the belief that the true data generating distribution is Gaussian. This condition, in fact, can be removed without affecting the result presented in Theorem~\ref{thm:Gauss}. Indeed, for any radius $\rho \ge 0$, the ambiguity set $\mbb B_{\rho}(\Pnom)$ by definition contains a Gaussian distribution with the same mean vector and covariance matrix with the nominal distribution $\Pnom$, and thus the feasible set of~\eqref{eq:Gauss-likelihood} is always non-empty and the value of the optimistic Gaussian log-likelihood is always finite. In Appendix~\ref{sect:gradient}, we provide the first- and second-order derivatives of the objective function of~\eqref{eq:gauss:refor2}, which can be exploited to derive efficient algorithm to solve the convex program~\eqref{eq:gauss:refor2}.

Returning to the construction of the classifier, one can now construct the classifier $\mc C_{\N}(x)$ using the optimistic Gaussian score ratio $\mc R_{\N}(x)$ expressed by
\begin{align}
    \mc R_{\mc N} (x) &\Let \frac{ \Sup{\QQ \in \mbb
B_{\rho_1}(\Pnom_1) \cap \N} \ell(x, \QQ)}{\Sup{\QQ \in \mbb
B_{\rho_0}(\Pnom_0) \cap \N} \ell(x, \QQ)}  \notag \\
&= \frac{\exp \Big(\half \Sup{\QQ \in \mbb
B_{\rho_1}(\Pnom_1) \cap \N} \mc L(x, \QQ) \Big)}{\exp \Big(\half \Sup{\QQ \in \mbb
B_{\rho_0}(\Pnom_0) \cap \N} \mc L(x, \QQ) \Big)}, \label{eq:ratio_gaussian}
\end{align}
where each nominal distribution $\Pnom_c$ is a Gaussian distribution with mean vector $\msa_c \in \R^d$ and covariance matrix $\covsa_c \in \PD^d$, and each ambiguity set $\mbb B_{\rho_c}(\Pnom_c)$ is defined as in~\eqref{eq:B-def}. Theorem~\ref{thm:Gauss} can be readily applied to evaluate the value $\mc R_{\N}(x)$, and classify $x$ using $\mc C_{\N}(x)$.
Furthermore, suppose that $\rho_0 = \rho_1 = 0$ and $\tau(x) = 1$, then the resulting classifier recovers the Quadratic Discriminant Analysis~\citep[Section~4.2.1]{ref:murphy2012machine}. The Bayesian version of the QDA can be equivalently obtained from $\mc C_{\N}$ by setting a proper value for $\tau(x)$.

It is imperative to elaborate on the improvement of Theorem~\ref{thm:Gauss} compared to the result reported in \citet[Section~3]{ref:nguyen2019calculating}. While both results are related to the evaluation of the optimistic Gaussian log-likelihood, \citet[Theorem~3.2]{ref:nguyen2019calculating} restricts the mean vector to its nominal value and optimizes only over the covariance matrix. On the other hand, Theorem~\ref{thm:Gauss} of this paper optimizes over both the mean vector and the covariance matrix, thus provides full flexibility to choose the optimal values of \textit{all} sufficient statistics of the family of Gaussian distributions. From a technical standpoint, the non-convexity is overcome in \citet[Theorem~3.2]{ref:nguyen2019calculating} through a simple change of variables; nonetheless, the proof of Theorem~\ref{thm:Gauss} demands an additional layer of duality arguments to disentangle the multiplicative terms between $\m$ and $\cov$ in both the objective function $\mc L$ and the divergence $\mathds D$. By inspecting the expressions in~\eqref{eq:gauss:refor1}, one can further notice that in general the optimal solution $\m\opt$ is distinct from the nominal mean $\msa$, this observation suggests that optimizing \textit{jointly} over $(\m, \cov)$ is indeed more powerful than optimizing simply over $\cov$ from a theoretical perspective.


\section{Numerical Results}
\label{sect:numerical}

All experiments are run on a standard laptop with 1.4 GHz Intel Core i5 and 8GB of memory, the codes and datasets are available at~\url{https://github.com/nian-si/bsc}.

\subsection{Decision Boundaries}
In this section, we visualize the classification decision boundaries generated by the classifiers $\mc C_{\text{nonparam}}$ proposed in Section~\ref{sect:nonparam} and $\mc C_{\N}$ proposed in Section~\ref{sect:param} using synthetic data. To ease the exposition, we consider a two dimensional feature space $d = 2$ and the class-conditional distributions are Gaussian of the form
\[
X | Y\!=\!0\!\sim\!\N\left(\begin{bmatrix} 0 \\ 1 \end{bmatrix}\!,\!I_{d} \right),
X | Y\!=\!1\!\sim\!\N\left(\begin{bmatrix} 1 \\ 0\end{bmatrix}\!,\!\begin{bmatrix}
  1 &\!0.5\\
  0.5 &\!1
\end{bmatrix} \right).
\]
We sample i.i.d.~data points $\{\wh x_{c, i}\}_{i = 1}^{n_c}$ in each $c \in \{0, 1\}$ with $n_0 = n_1=1000$  as the training set, then estimate the nominal mean $\msa_c$ and the nominal covariance matrix $\covsa_c$ for each class $c \in \{0, 1\}$ using the sample average formula~\eqref{eq:sample-average}.

We first consider when the ambiguity sets have the same radius, i.e.,  $\rho_0 =\rho_1=\wh\rho$ and fix the threshold $\tau(x) = 1$ for every $x$. Figure \ref{fig:same_rho} shows the optimistic Gaussian and nonparametric decision boundaries for $\wh \rho \in \{0.5,0.7\}$.  We find that for optimistic Gaussian decision rule, the decision boundaries look similar across different radii; while in nonparametric case, the decision boundaries exhibit different shapes. We then consider the case with distinct radii by setting $\vec{\rho} = (\rho_0,\rho_1)= (0.1,1.0)$ and $\vec{\rho} = (\rho_0,\rho_1)= (1.0,0.1)$. Further, we fix the threshold to a constant $\tau(x)  = \tau\opt$ for a scalar $\tau\opt \in \R_+$ that solves
\be
\Max{\tau \geq 0} \sum_{i=1}^{n_0} \mathbf{1} \{\mc R(\wh x_{0, i})\!<\!\tau\}\!+\!\sum_{i=1}^{n_1} \mathbf{1} \{\mc R(\wh x_{1, i})\!\geq\!\tau\} ,
\label{eqn:tune_gamma}
\ee
where $\mathbf{1}\{ \cdot \}$ is the indicator function.
The decision boundaries are plotted in Figure \ref{fig:different_rho}. We find the decision boundaries have different shapes for different decision rules and for different choices of radii.

\begin{figure}[!ht]
\centering
\subfigure[Gaussian, $\wh \rho = 0.5$]{
\label{pic:kl_0.5} \includegraphics[width=1.51in]{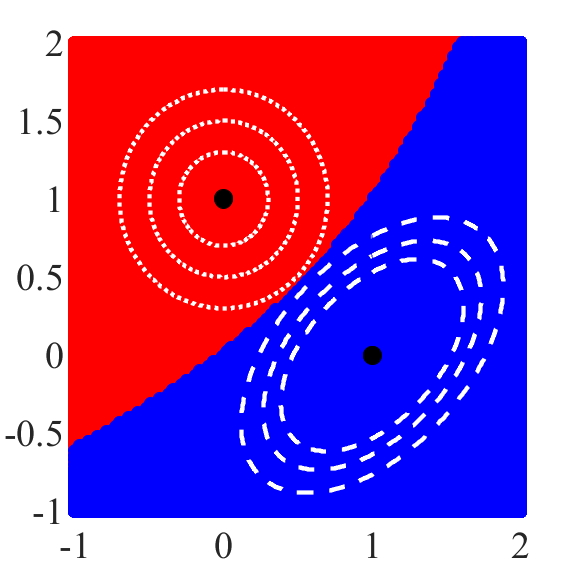}}
\subfigure[Nonparametric, $\wh \rho = 0.5$]{
\label{pic:non_para_0.5} \includegraphics[width=1.51in]{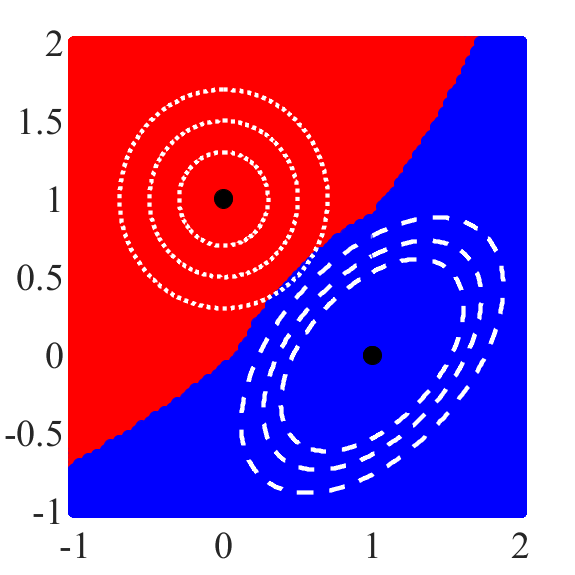}}
\subfigure[Gaussian, $\wh \rho = 0.7$]{
\label{pic:kl_1} \includegraphics[width=1.51in]{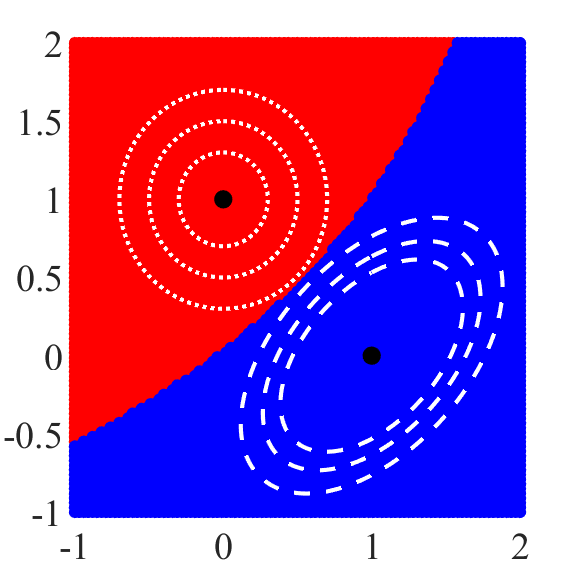}}
\subfigure[Nonparametric, $\wh \rho = 0.7$]{
\label{pic:non_para_1} \includegraphics[width=1.51in]{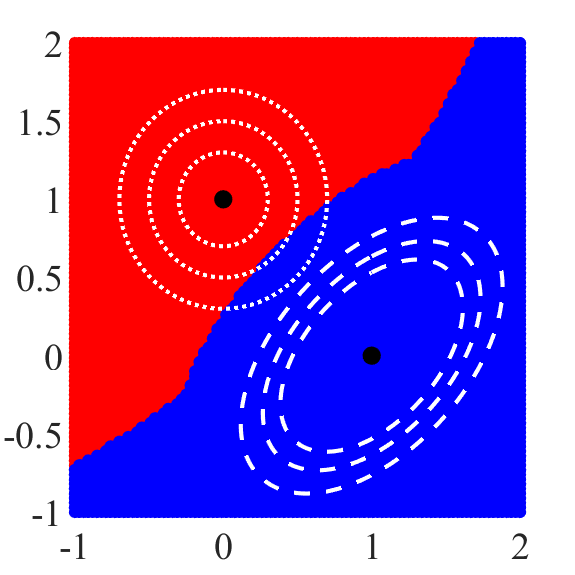}}
\caption{Decision boundaries for different $\wh \rho$. Red/blue regions indicate the class partitions, black dots locate the mean, and white dashed ellipsoids draw  the class-conditional density contours.}
\label{fig:same_rho}
\vspace{-3mm}
\end{figure}

\begin{figure}[ht]
\centering
\subfigure[\scriptsize Gaussian, $\vec{\rho}\!=\!(0.1,1.0)$]{
\label{pic:kl_0.1_1.0} \includegraphics[width=1.51in]{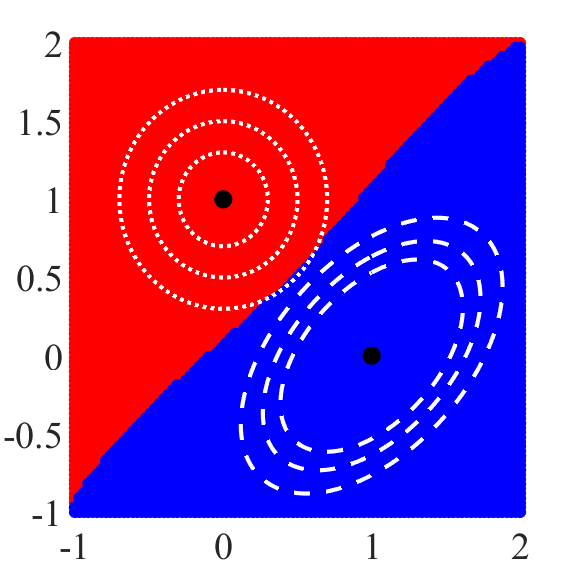}}
\subfigure[\scriptsize Nonparametric, $\vec{\rho}\!=\!(0.1,1.0)$]{
\label{pic:non_para_0.1_1.0} \includegraphics[width=1.51in]{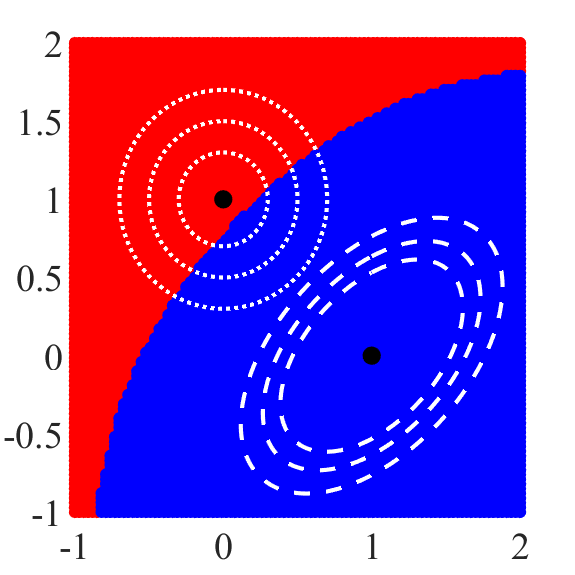}} \vspace{-2mm}
\subfigure[\scriptsize Gaussian,  $\vec{\rho}\!=\!(1.0,0.1)$]{
\label{pic:kl_1.0_0.1} \includegraphics[width=1.51in]{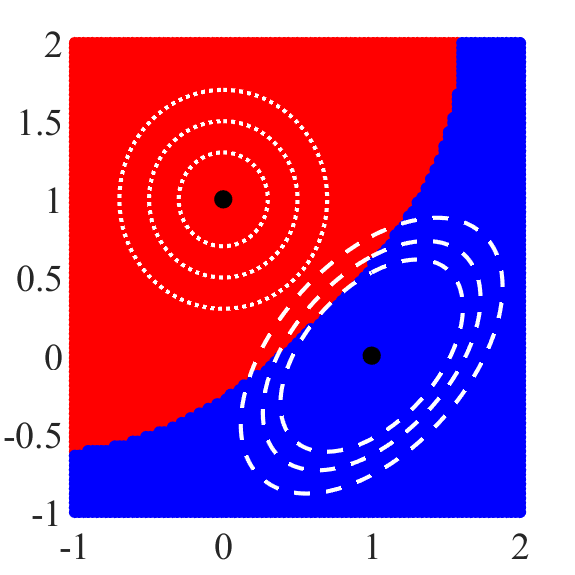}}
\subfigure[\scriptsize Nonparametric, $\vec{\rho}\!=\!(1.0,0.1)$]{
\label{pic:non_para_1.0_0.1} \includegraphics[width=1.51in]{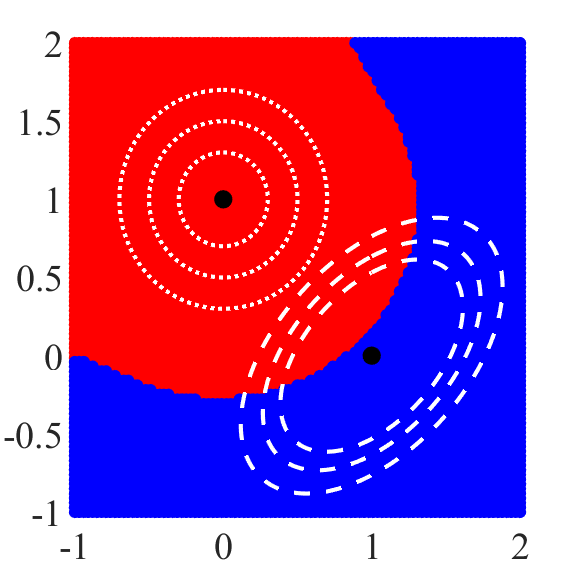}}
\caption{Decision boundaries with distinct radii. Indications are verbatim from Figure~\ref{fig:same_rho}.}
\label{fig:different_rho}
\vspace{-6mm}
\end{figure}

\subsection{Real Data Experiments}
\begin{table*}[!ht]
  \centering
  \caption{Correct classification
rate on the benchmark date sets. \textbf{Bold} number corresponds to the best performance in each dataset.}
  \vskip 0.1in
  \begin{center}
\begin{small}
\begin{sc}
    \begin{tabular}{lccccccc}
    \toprule
    Dataset & GQDA, CV & GQDA, CLT & NPQDA, CV & NPQDA, CLT & KQDA  & RQDA  & SQDA \\
    \midrule
    australian & 85.38 & 84.91 & \textbf{85.84} & 85.03 & 85.38 & 85.61  & 85.37 \\
    banknote   & \textbf{99.83} & 99.33 & 99.3  & \textbf{99.83} & 99.8  & 99.77 & \textbf{99.83} \\
    climate model & \textbf{94.81} & 89.33 & 93.92 & 90.37 & 92.59  & 94.07 & 93.92 \\
    cylinder & 71.04 & 70.81 & \textbf{71.11} & 70.89  & \textbf{71.11} & \textbf{71.11} & 70.67 \\
    diabetic & 75.52 & 73.49 & \textbf{76.09} & 76.30  & 75.47 & 75.00        & 74.32 \\
    fourclass & 77.82 & 79.26 & \textbf{80.28} & 78.98 & 78.38  & 79.07 & 78.84 \\
    haberman & 74.93 & 75.33 & \textbf{75.45} & \textbf{75.45} & 74.94 & 74.80    & 74.16 \\
    heart & 81.76 & 83.09 & 83.09 & 81.91 & 81.76 &  81.91 & \textbf{83.68} \\
    housing & 91.66 & 90.55 & 91.81 & 91.50  & 91.66  & 91.66 & \textbf{91.97} \\
    ilpd  & 68.84 & \textbf{69.52} & 69.25 & 68.15 & 69.18 &  67.94 & \textbf{69.52} \\
    mammographic mass & \textbf{80.39} & 80.00    & 79.90  & 79.61 & 79.95 &  80.05 & 80.24 \\
    \bottomrule
    \end{tabular}%
    \end{sc}
\end{small}
\end{center}
\vskip -0.1in
  \label{tab:acc}%
  \vspace{-5mm}
\end{table*}%

In our experiments, we first compute the nominal mean by empirical average and we use the Ledoit-Wolf covariance estimator \cite{ref:ledoi2004well} to compute a well-conditioned nominal covariance matrix. We experiment two methods of tuning the radii $\rho$ of the ambiguity sets: using cross-validation on training data, or using the quantile of the limiting distribution in Corollary \ref{cor:gaussian}. Specifically, for the second criteria, we choose
\[\rho_c = n_c^{-1} \chi^2_\alpha(d(d+3)/2) \qquad \forall c \in \{0,1\}, \]
where $n_c$ is the number of training samples in class $c$ and $\chi^2_\alpha(d(d+3)/2)$ is the $\alpha$-quantile of the chi-square distribution with $d(d+3)/2$ degrees of freedom. Notice that for large degrees of freedom, the chi-square distribution concentrates around the mean, because a chi-square random variable with $k$ degrees of freedom is the sum of $k$ i.i.d.~$\chi^2(1)$. The optimal asymptotic value of the radius $\rho_c$ is therefore insensitive to the choice of $\alpha$, so we select numerically $\alpha = 0.5$ in our experiments. We tune the threshold to maximize the training accuracy following \eqref{eqn:tune_gamma} after computing the ratio value for each training sample.
The whole procedure is summarized in Algorithm~\ref{alg}. In particular, this algorithm trains the parameters using only one pass over the training samples, which makes it significantly faster than the cross-validation approach. We observe empirically in most cases that the performance of classifying using Algorithm~\ref{alg} is comparable in terms of test accuracy to classifying with cross-validating on the tuning parameters.

\begin{algorithm}[t]
 \caption{Optimistic score ratio classification}
 \begin{algorithmic}[1]
  \STATE \textbf{Input:} datasets $\{\wh x_{c,i}\}_{i=1}^{n_c}$ for $c \in \{0, 1\}$. A test data $x$.
     \STATE Compute the nominal mean and the nominal covariance matrix
  \STATE Compute the radii $\rho_c \gets  n_c^{-1} \chi^2_{0.5}(d(d+3)/2)$.
  \STATE Compute the optimistic ratio $\mc R(\wh x_{c,i})$ for every $\wh x_{c, i}$
  \STATE Compute the threshold $\tau\opt$ that solves \eqref{eqn:tune_gamma}.
  \STATE \textbf{Output:} classification label  $\mathbf{1} \{\mc R(x)\geq\tau\}$.
 \end{algorithmic}
 \label{alg}
\end{algorithm}

We test the performance of our classification rules on various datasets from the UCI repository \cite{Dua:2019}. Specifically, we compare the following methods:
\begin{itemize}[wide]
    \item Gaussian QDA (\textbf{GQDA}) and Nonparametric QDA (\textbf{NPQDA}): Our classifiers $\mc C_{\N}$ and $\mc C_{\text{nonparam}}$;
    \item Kullback-Leibler QDA (\textbf{KQDA}): The classifier based on KL ambiguity sets with fixed mean \cite{ref:nguyen2019calculating};

    \item Regularized QDA (\textbf{RQDA}): The regularized QDA based on the linear shrinkage covariance estimator $\wh \Sigma_c+\rho_c I_d$;
    \item Sparse QDA (\textbf{SQDA}): The sparse QDA based on the graphical lasso covariance estimator \cite{ref:friedman2008sparse} with parameter $\rho_c$.
\end{itemize}
For \textbf{GQDA} and \textbf{NPQDA}, we also compare the performance of different strategies to choose the radii $\rho$ using cross-validation (CV) and selection based on Theorem \ref{thm:asymptotic} (CLT). For all methods that need cross-validation, we randomly select 75\% of the data for training and the remaining 25\%
for testing. The size of the ambiguity sets and the regularization parameter are selected using stratified 5-fold cross-validation. Furthermore, to promote a fair comparison, we tune the threshold for every method using \eqref{eqn:tune_gamma}. The performance of the classifiers is measured by the  average \textit{correct classification
rate} (CCR) on the validation set. The average CCR score
over 10 trials are reported in Table~\ref{tab:acc}.

\onecolumn

\appendix
\renewcommand\thesection{\Alph{section}}
\renewcommand{\theequation}{A.\arabic{equation}}

\section{Proof of Section~\ref{sect:intro}}

\begin{proof}[Proof of Proposition~\ref{prop:optimal-action}]
To ease the exposition, we use the following notational shorthands $\mbb B_c = \mbb B_{\rho_c}(\Pnom_c) \cap \mc P$ for $c \in \{0, 1\}$ and
\[
    f_c^{\max}(x) = \Sup{f_c \in \mbb B_c}~f_c(x), \quad f_c^{\min}(x) = \Inf{f_c \in \mbb B_c}~f_c(x) \qquad \forall c \in \{0, 1\}.
\]
If action $a = 1$ is chosen, then the worst-case probability of mis-classification is
\begin{subequations}
\begin{align}
   \Sup{\PP \in \mc B}~\PP( Y = 0 | X = x) &=
   \left\{
       \begin{array}{cl}
       \sup & \ds \frac{ f_0(x) \pi_0} {f_0(x) \pi_0 + f_1(x) \pi_1} \\
       \st & f_0 \in \mbb B_0,~f_1 \in \mbb B_1
       \end{array}
   \right. \notag \\
   &=
   \Sup{f_0 \in \mbb B_0}~\Sup{f_1 \in \mbb B_1}~\ds \frac{ f_0(x) \pi_0} {f_0(x) \pi_0 + f_1(x) \pi_1} \notag \\
   &= \Sup{f_0 \in \mbb B_0}~\ds \frac{ f_0(x) \pi_0} {f_0(x) \pi_0 + f_1^{\min}(x) \pi_1} \label{eq:2layer-1}\\
    &= \frac{f_0^{\max}(x) \pi_0}{f_0^{\max}(x) \pi_0 + f_1^{\min}(x) \pi_1}, \label{eq:2layer-2}
\end{align}
\end{subequations}
where equality~\eqref{eq:2layer-1} holds because $\pi_1 > 0$, thus for any $f_0 \in \mbb B_0$, the optimal choice of $f_1$ for the inner supremum problem will minimize $f_1(x)$ over all $f_1 \in \mbb B_1$. Equality~\eqref{eq:2layer-2} holds because $f_1^{\min}(x) \pi_1 > 0$, thus it is optimal to choose $f_0$ that maximizes $f_0(x)$ over all $f_0 \in \mbb B_0$. Using similar lines of arguments, if action $a = 0$ is chosen, then the worst-case probability of mis-classification is
\begin{align*}
   \Sup{\PP \in \mc B}~\PP( Y = 1 | X = x)  &=
   \left\{
   \begin{array}{cl}
   \sup & \ds \frac{ f_1(x) \pi_1} {f_0(x) \pi_0 + f_1(x) \pi_1} \\
   \st & f_0 \in \mbb B_0,~f_1 \in \mbb B_1
   \end{array}
   \right. \\
   &= \frac{f_1^{\max}(x) \pi_1}{f_0^{\min}(x) \pi_0 + f_1^{\max}(x) \pi_1}.
\end{align*}
Thus, by comparing the two values of the worst-case probability, action $a=1$ is optimal whenever
\[
    \frac{f_0^{\max}(x) \pi_0}{f_0^{\max}(x) \pi_0 + f_1^{\min}(x) \pi_1} \leq \frac{f_1^{\max}(x) \pi_1}{f_0^{\min}(x) \pi_0 + f_1^{\max}(x) \pi_1},
\]
which in turn is equivalent to the condition
\begin{equation} \label{eq:ratio-original}
\frac{f_1^{\max}(x)}{{f_0^{\max}(x)}} \geq \frac{f_0^{\min}(x)\pi_0^2}{{f_1^{\min}(x)\pi_1^2}}.
\end{equation}
By setting the right-hand side of~\eqref{eq:ratio-original} to a threshold $\tau(x) \in \R_+$, we arrive at the postulated result.
\end{proof}
As the proof reveals in~\eqref{eq:ratio-original}, the optimal threshold $\tau(x)$ in the statement of Proposition~\ref{prop:optimal-action} admits an explicit expression
\[
    \tau(x) =\frac{f_0^{\min}(x)\pi_0^2}{{f_1^{\min}(x)\pi_1^2}}.
\]
This value $\tau(x)$ can be found by evaluating the minimum likelihood values $f^{\min}_0(x)$ and $f^{\min}_0(x)$. Unfortunately, it remains intractable to find the exact values of $f^{\min}_0(x)$ and $f^{\min}_1(x)$. To demonstrate this fact, we consider the Gaussian parametric setting as in Section~\ref{sect:param}, and evaluating the minimum likelihood in this case is equivalent to solving
\be \label{eq:pess}
\begin{array}{cl}
\min & ~ -(\m - x)^\top \cov^{-1} (\m - x) - \log\det \cov \\
\st& \m \in \R^d,~  \cov \in \PD^d \\
&(\m - \msa)^\top \cov^{-1} (\m - \msa) + \Tr{\covsa \cov^{-1}} - \log\det \covsa \cov^{-1} - d \leq  \rho
\end{array}%
\ee
for some $\msa \in \R^d$, $\covsa \in \PD^d$ and $\rho \ge 0$. Problem~\eqref{eq:pess} is the minimization counterpart of the maximization problem~\eqref{eq:gauss:refor0}, it is also non-convex, however, we are not aware of any tractable approach to solve~\eqref{eq:pess}.

\section{Proofs of Section~\protect\ref{sect:ambiguity}}

\begin{proof}[Proof of Theorem~\protect\ref{thm:asymptotic}]
Throughout this proof, we use $\xrightarrow{dist.}$ and $\xrightarrow{p.}$
to denote the convergence in distribution and in probability, respectively.
For $n$ sufficiently big, $\covsa_{n}$ defined as in~%
\eqref{eq:sample-average} is invertible with probability 1. In this case, we
find
\begin{align}
\mathds D\big((\msa_{n},\covsa_{n})\parallel (m,S)\big)& =\Tr{\covsa_n
S^{-1}}+(m-\msa_{n})^{\top }S^{-1}(m-\msa_{n})-d-\log \det (\covsa_{n}S^{-1})
\notag \\
& =\frac{1}{n}\sum_{t=1}^{n} \wh \eta_{t}^{\top } \wh \eta_{t}-d-\log \det
(S^{-\half}\covsa_{n}S^{-\half}),  \label{eq:D_re}
\end{align}%
where $\wh \eta_{t}=S^{-\half}(\wh \xi_{t}-m)$ is the isotropic
transformation of $\wh \xi_{t}$ for each $t=1,\ldots ,n$. Furthermore,
denote by $\bar{\mu}_{n}$ the sample average of $\wh \eta_1, \ldots, \wh %
\eta_n$ defined as
\begin{equation*}
\bar{\mu}_{n}=\frac{1}{n}\sum_{t=1}^{n} \wh \eta_{t}=S^{-\half}(\msa_{n}-m).
\end{equation*}%
%
%
%
%
%
%
%
%
By adding and subtracting $\log\det \left(n^{-1} \sum_{t=1}^{n} \wh \eta
_{t} \wh \eta _{t}^{\top } \right)$ into~\eqref{eq:D_re}, we have
\begin{align*}
&\mathds D\big((\msa_{n},\covsa_{n})\parallel (m,S)\big) = \\
&\underbrace{%
\left(\log \det \big(\frac{1}{n} \sum_{t=1}^{n} \wh \eta _{t} \wh \eta
_{t}^{\top } \big) -\log \det (S^{-\half}\covsa_{n} S^{-\half})\right)}_%
\text{(A)} + \underbrace{\left(\frac{1}{n}\sum_{t=1}^{n} \wh \eta_{t}^{\top
} \wh \eta_{t} - d - \log \det \big(\frac{1}{n} \sum_{t=1}^{n} \wh \eta_{t} %
\wh \eta_{t}^{\top } \big)\right)}_\text{(B)}.
\end{align*}
We analyze the 2 terms (A) and (B) separately. First, rewrite
\begin{align*}
S^{-\half}\covsa_{n}S^{-\half}& =S^{-\half}\Big(\frac{1}{n}\sum_{t=1}^{n}(%
\wh \xi_{t}-\msa_{n})(\wh \xi_{t}-\msa_{n})^{\top }\Big)S^{-\half} \\
& =S^{-\half}\big(\frac{1}{n}\sum_{t=1}^{n}(\wh \xi_{t}-m+m-\msa_{n})(\wh %
\xi_{t}-m+m-\msa_{n})^{\top }\big)S^{-\half} \\
& =\frac{1}{n}\sum_{t=1}^{n}\big(\wh \eta_{t} \wh \eta_{t}^{\top }+S^{-\half%
}(m-\msa_{n}) \wh \eta_{t}^{\top } + \wh \eta_{t}(m-\msa_{n})^{\top }S^{-%
\half} + S^{-\half}(m-\msa_{n})(m-\msa_{n})^{\top }S^{-\half}\big) \\
& =\frac{1}{n}\sum_{t=1}^{n}\big(\wh \eta_{t} \wh \eta_{t}^{\top }-\bar{\mu}%
_{n} \wh \eta_{t}^{\top } - \wh \eta_{t}\bar{\mu}_{n}^{\top }+\bar{\mu}_{n}%
\bar{\mu}_{n}^{\top }\big) \\
& = \left(\frac{1}{n}\sum_{t=1}^{n} \wh \eta_{t} \wh \eta_{t}^{\top }\right)-%
\bar{\mu}_{n}\bar{\mu}_{n}^{\top }.
\end{align*}%
%
%
%
%
%
%
%
%
Then, (A) becomes
\begin{align*}
\log \det \big(\frac{1}{n} \sum_{t=1}^{n} \wh \eta_{t} \wh \eta_{t}^{\top } %
\big) -\log \det (S^{-\half}\covsa_{n}S^{-\half})= -\log \det \left(I_{d}
-\left(\frac{1}{n}\sum_{t=1}^{n} \wh \eta_{t} \wh \eta_{t}^{\top
}\right)^{-1} \left(\bar{\mu}_{n}\bar{\mu}_{n}^{\top } \right)\right).
\end{align*}
By the weak law of large numbers, as $n\uparrow \infty $, we find
\begin{equation*}
\frac{1}{n} \sum_{t=1}^{n} \wh \eta_{t} \wh \eta_{t}^{\top }\xrightarrow{p.}%
I_{d}\quad \text{and} \quad \frac{1}{n}\sum_{t=1}^{n} \wh \eta_{t}%
\xrightarrow{p.}0.
\end{equation*}%
By the central limit theorem, we find as $n\uparrow \infty $
\begin{equation*}
\sqrt{n}\bar{\mu}_{n}\xrightarrow{dist.} H ,\quad \sqrt{n}\left( \frac{1}{n}%
\sum_{t=1}^{n} \wh \eta_{t} \wh \eta_{t}^{\top }-I_{d}\right) %
\xrightarrow{dist.}Z,
\end{equation*}%
where the random vector $H$ and the random matrix $Z$ are defined as in the
statement of the theorem.
By Slutsky's theorem~\citep[Theorem~2.8]{ref:vaart1998asymptotic}, we find
\begin{equation*}
n\left(\frac{1}{n}\sum_{t=1}^{n} \wh \eta _{t} \wh \eta _{t}^{\top
}\right)^{-1} \left(\bar{\mu}_{n}\bar{\mu}_{n}^{\top } \right) %
\xrightarrow{dist.} HH^{\top}.
\end{equation*}
By the delta method~\citep[Theorem~3.1]{ref:vaart1998asymptotic}, we have
\begin{equation}
n \times \Big( \underbrace{\log \det \big(\frac{1}{n} \sum_{t=1}^{n} \wh %
\eta_{t} \wh \eta_{t}^{\top } \big) -\log \det (S^{-\half}\covsa _{n}S^{-%
\half})}_\text{(A)}\Big) \xrightarrow{dist.} \Tr{HH^{\top}} = H^{\top}H.
\label{eq:(1)}
\end{equation}
Now, we are ready to analyze (B). Using a Taylor expansion for the
log-determinant function around $I_d$, we find
\begin{align*}
&\log \det \big(\frac{1}{n} \sum_{t=1}^{n} \wh \eta _{t} \wh \eta_{t}^{\top
} \big) \\
= &\log\det(I_d) + \text{Tr}\left[\left( \frac{1}{n}\sum_{t=1}^{n} \wh \eta
_{t} \wh \eta_{t}^{\top }-I_{d}\right) \right]- \frac{1}{2} \text{Tr}\left[%
\left( \frac{1}{n}\sum_{t=1}^{n} \wh \eta_{t} \wh \eta_{t}^{\top
}-I_{d}\right)^2 \right] + o\left(\text{Tr}\left[\left( \frac{1}{n}%
\sum_{t=1}^{n} \wh \eta_{t} \wh \eta_{t}^{\top } - I_{d}\right)^2\right]
\right).
\end{align*}
Therefore, by the second-order delta method, we have
\begin{equation}
n\times\Big( \underbrace{\frac{1}{n}\sum_{t=1}^{n} \wh \eta_{t}^{\top } \wh %
\eta_{t}-d-\log \det \big(\frac{1}{n} \sum_{t=1}^{n} \wh \eta_{t} \wh %
\eta_{t}^{\top } \big)}_{\text{(B)}}\Big) \xrightarrow{dist.} \frac{1}{2}
\text{Tr}\left[Z^2 \right].  \label{eq:(2)}
\end{equation}
Finally, by combining the limits from~\eqref{eq:(1)} and \eqref{eq:(2)}, we
obtain the postulated result. 
\end{proof}

\begin{proof}[Proof of Corollary~\protect\ref{cor:gaussian}]
From Theorem \ref{thm:asymptotic}, we have as $n\uparrow \infty $
\begin{equation*}
n\times \frac{1}{2} \KL\big(\mc N(\msa_{n},\covsa_{n}) \parallel \mc N (m, S) \big) =
n\times \mathds D\big((\msa_{n},\covsa_{n})\parallel (m,S)\big)\rightarrow
H^{\top }H+\frac{1}{2}\Tr{Z^2}\quad \text{in distribution,}
\end{equation*}%
where the random vector $H$ and the random matrix $Z$ are defined as in the
statement of Theorem~\ref{thm:asymptotic}. In the Gaussian setting, the
elements of the isotropic random vector $\eta$ are i.i.d.~standard
univariate normal random variables. Therefore, we have%
\begin{equation*}
\mathrm{cov}(Z_{jk},Z_{j^{\prime }k^{\prime }})=
\begin{cases}
\EE_{\PP}\left[ \left( \eta_j\right) ^{4}\right] -1 & \text{if }
j=k=j^{\prime }=k^{\prime }, \\
1 & \text{if } j<k,\left( j=j^{\prime },k=k^{\prime }\text{ or }j=k^{\prime
},j^{\prime }=k\right), \\
0 & \text{otherwise.}%
\end{cases}%
\end{equation*}%
Recall that $\EE_{\PP}\left[ \left( \eta_j\right) ^{4}\right] =3,$ which
gives $\mathrm{cov}(Z_{jj},Z_{jj})=2.$ Hence, $\frac{1}{2}\Tr{Z^2}$
follows $\chi ^{2}\left( d(d+1)/2\right) $ and $H^{\top }H$ follows $\chi
^{2}\left( d\right) $. Finally, since $H$ and $Z$ are independent in the
Gaussian case, we have $H^{\top }H + \frac{1}{2}\Tr{Z^2}$ follows $\chi
^{2}\left( d\right) +\chi ^{2}\left( d(d+1)/2\right) =\chi ^{2}\left(
d(d+3)/2\right)$.
\end{proof}


\section{Proofs of Section~\protect\ref{sect:nonparam}}

We first prove the compactness property of the uncertainty set $\mc U_\rho(\msa, \covsa)$.

\begin{lemma}[Compactness of $\mc U_\rho(\msa, \covsa)$] \label{lemma:U-compact}
For any $\msa \in \R^d$, $\covsa \in \PD^d$ and $\rho \in \R_+$, the set $\mc U_\rho(\msa, \covsa)$ written as
\begin{align*}
	\mc U_\rho(\msa, \covsa) = \{ (\m, \cov) \in \R^d \times \PD^d:  \Tr{\covsa \cov^{-1}} - \log\det (\covsa \cov^{-1}) - d + (\m - \msa)^\top \cov^{-1} (\m - \msa) \le \rho \}
\end{align*}
is compact.
\end{lemma}

\begin{proof}[Proof of Lemma~\protect\ref{lemma:U-compact}]
If $\rho = 0$ then $\mc U_\rho(\msa, \covsa)$ is a singleton $\{(\msa, \covsa%
)\}$ and the claim holds trivially. For the rest of the proof, we consider
when $\rho > 0$. Pick an arbitrary $(\m, \cov) \in \U_\rho(\msa, \covsa)$,
it is obvious that $\cov$ should satisfy
\begin{equation*}
\Tr{\covsa^\half \cov^{-1} \covsa^\half} - \log\det (\covsa^\half \cov^{-1} %
\covsa^\half) - d \leq \rho,
\end{equation*}
which implies that $\cov$ is bounded. To see this, suppose that $\{\cov%
_{k}\}_{k \in \mbb N}$ is a sequence of positive definite matrices and $%
\{\sigma_k\}_{k \in \mbb N}$ is the corresponding sequence of the minimum
eigenvalues of $\{\covsa^{-\half} \cov_k^{-1} \covsa^{-\half}\}_{k \in \mbb %
N}$. Because the function $\sigma \mapsto \sigma - \log \sigma - 1$ is
non-negative for every $\sigma > 0$, we find
\begin{equation*}
\Tr{\covsa^\half \cov_k^{-1} \covsa^\half} - \log\det (\covsa^\half \cov%
_k^{-1} \covsa^\half) - d \geq \sigma_k - \log \sigma_k -1.
\end{equation*}
If $\cov_k$ tends to infinity, then $\sigma_k$ tends to 0, and
in this case $\sigma_k - \log \sigma_k -1 \rightarrow +\infty$. This implies
that $\cov$ should be bounded in the sense that $\cov \preceq \bar \sigma
I_d $ for some finite positive constant $\bar \sigma$. Using an analogous
argument, we can show that $\cov$ is lower bounded in the sense that $\cov %
\succeq \underline{\sigma} I_d$ for some finite positive constant $%
\underline{\sigma}$. As a consequence, $\m$ is also bounded because $\m$
should satisfy $\underline{\sigma} \|\m - \msa \|_2^2 \leq \rho$. We now can
rewrite $\mc U_\rho(\msa, \covsa)$ as
\begin{equation*}
\mc U_\rho(\msa, \covsa) = \{ (\m, \cov) \in \R^d \times \PD^d: \underline{%
\sigma} \|\m - \msa \|_2^2 \leq \rho,~\underline{\sigma} I_d \preceq \cov %
\preceq \bar{\sigma} I_d,~\mathds D \big( (\msa, \covsa) \parallel (\m, \cov%
) \big) \leq \rho \},
\end{equation*}
which implies that $\mc U_\rho(\msa, \covsa)$ is a closed set because $%
\mathds D \big( (\msa, \covsa) \parallel (\cdot, \cdot) \big)$ is a
continuous function over $(\m, \cov)$ when $\cov$ ranges over $\underline{%
\sigma} I_d \preceq \cov \preceq \bar{\sigma} I_d$. This observation coupled
with the boundedness of $(\m, \cov)$ established previously completes the
proof.
\end{proof}

For a fixed $\msa \in \R^d$, $x \in \R^d$ and $\eps \in \R_+$, define the
following function $g : \PD^d \to \R_+$ as \be \label{eq:g-def} g(\Omega) %
\Let \left\{
\begin{array}{cl}
\min & (\m - x)^\top \Omega (\m - x) \\
\st & \m \in \R^d,~(\m - \msa)^\top \Omega (\m - \msa) \leq \eps,%
\end{array}
\right. \ee
where the dependence of $g$ on $\msa$, $x$ and $\eps$ has been made implicit
to avoid clutter. The objective function of problem~\eqref{eq:g-def} is continuous in $\m$ and the feasible set of problem~\eqref{eq:g-def} is compact because $\Omega \in \PD^d$, which justify the minimization operator of problem~\eqref{eq:g-def}. The next lemma asserts that the value $g(\Omega)$ coincides with the
optimal value of a univariate convex optimization problem.
\begin{lemma}[Reformulation of $g$] \label{lemma:g-refor}
    For any $\msa \in \R^d$, $x \in \R^d$, $\eps \in \R_+$ and $\Omega \in \PD^d$, the value $g(\Omega)$ coincides with the optimal value of the univariate convex optimization problem
    \be \label{eq:g-refor}
     \Max{\lambda \ge 0} \left\{ \frac{\lambda}{1+\lambda} (x - \msa)^\top \Omega (x -\msa) - \lambda \eps \right\}.
    \ee
    Moreover, denote by $\lambda\opt$ the unique optimal solution of the maximization problem~\eqref{eq:g-refor}, then the unique minimizer $\m\opt$ of problem~\eqref{eq:g-def} is $\m\opt = (x + \lambda\opt \msa)/(1 + \lambda\opt)$. Furthermore, we have
    \[
        \left\{
        \begin{array}{ll}
            \lambda\opt = 0,~g(\Omega) = 0 & \text{if } \eps \ge (x - \msa)^\top \Omega (x - \msa), \\
            \lambda\opt = \big(\sqrt{\eps (x-\msa)^\top \Omega (x - \msa)} - \eps\big)/\eps,~g(\Omega) = \big(\sqrt{\eps} - \sqrt{(x - \msa)^\top \Omega (x - \msa)} \big)^2 & \text{otherwise.}
        \end{array}
        \right.
    \]
\end{lemma}

\begin{proof}[Proof of Lemma~\protect\ref{lemma:g-refor}]
Using a change of variables $y \leftarrow \m - \msa$ and a change of
parameters $w \leftarrow x - \msa$, problem~\eqref{eq:g-def} can be recast
in the following equivalent form \be \label{eq:g-primal} g(\Omega) = \left\{
\begin{array}{cl}
\min & (y - w)^\top \Omega (y - w) \\
\st & y \in \R^d,~y^\top \Omega y \leq \eps,%
\end{array}
\right. \ee
which is a convex optimization problem. Assume momentarily that $\eps > 0$.
By invoking a duality argument, we find
\begin{subequations}
\begin{align}
g(\Omega) &= \Min{y} \Max{\lambda \ge 0}~ (y - w)^\top \Omega (y - w) +
\lambda \big( y^\top \Omega y - \eps \big)  \notag \\
&=\Max{\lambda \ge 0} ~ w^\top \Omega w -\lambda \eps + \Min{y}~\{
(1+\lambda) y^\top \Omega y - 2y^\top \Omega w \}  \label{eq:g-refor-1} \\
&= \Max{\lambda \ge 0} ~ \frac{\lambda}{1 + \lambda} w^\top \Omega w
-\lambda \eps ,  \label{eq:g-refor-2}
\end{align}
where the interchanging of the inf-sup operators are justified because the
feasible set of the primal problem~\eqref{eq:g-primal} is non-empty and
compact~\citep[Proposition~5.5.4]{ref:bertsekas2009convex}. For any $\lambda
\ge 0$, the minimizer of the inner minimization problem in~%
\eqref{eq:g-refor-1} is
\end{subequations}
\begin{equation*}
y\opt(\lambda) = \frac{w}{1 + \lambda}.
\end{equation*}
Furthermore, this minimizer $y\opt(\lambda)$ is unique for any $\lambda \ge
0 $ because the objective function of the inner minimization over $y$ in~%
\eqref{eq:g-refor-1} is strictly convex in $y$. Substituting this optimal
solution into the objective of~\eqref{eq:g-refor-1} leads to~%
\eqref{eq:g-refor-2}, and substituting the value of $w$ by $x - \msa$ leads
to the reformulation~\eqref{eq:g-primal}.

We now study the maximizer $\lambda\opt$ of problem~\eqref{eq:g-refor-2}.
The Karush-Kuhn-Tucker condition asserts that there exists $\dualvar\opt \in %
\R_+$ such that $(\lambda\opt, \dualvar\opt)$ satisfy the system of algebraic equations
\begin{equation*}
\left\{
\begin{array}{rcl}
(1+\lambda\opt)^{-2} w^\top \Omega w - \dualvar\opt & = & \eps \\
\dualvar\opt \lambda\opt & = & 0 \\
\dualvar\opt~\ge~ 0,~ \lambda\opt & \ge & 0.%
\end{array}
\right.
\end{equation*}
If $w^\top \Omega w \leq \eps$, then $\lambda\opt = 0$. If $w^\top \Omega w
> \eps$, then
\[
\lambda\opt = \sqrt{\frac{w^\top \Omega w}{\eps}} - 1.
\]
In both cases, $\lambda\opt$ is unique. Substituting the value of $\lambda%
\opt$ into the objective function of~\eqref{eq:g-refor-2} gives the
analytical expression for $g(\Omega)$.

We note that when $\eps = 0$, we have $g(\Omega) = (x - \msa)^\top \Omega (x
- \msa)$. The expressions for $\lambda\opt$ remain still valid in this case
by taking the limit as $\eps \downarrow 0$. Finally, the uniqueness of $\m%
\opt$ follows from the uniqueness of $\lambda\opt$ and $y\opt(\lambda)$
obtained previously. The proof is thus completed.
\end{proof}

We are now ready to prove Theorem~\ref{thm:nonparam} in the main text.

\begin{proof}[Proof of Theorem~\ref{thm:nonparam}]
The optimistic nonparametric score evaluation problem can be decomposed using a two-layer formulation~\eqref{eq:two-layer} as
\begin{equation*}
\Sup{\QQ \in \mbb B_\rho(\Pnom)}~\QQ(\{x\}) = \Sup{(\m, \cov) \in \mc
U_\rho(\msa, \covsa)} \Sup{\QQ \in \mc M(\m, \cov)}~\QQ(\{x\}).
\end{equation*}
Using the result from~\citet{ref:marshall1960multivariate} or~\citet[Theorem~6.1]{ref:bertsimas2005optimal} to reformulate the inner supremum
problem, we have
\begin{equation*}
\Sup{\QQ \in \mc M(\m, \cov)}~\QQ(\{x\}) = \frac{1}{1+ (\m - x)^\top \cov%
^{-1} (\m - x)},
\end{equation*}
where the supremum is attained thanks to~\citet[Theorem~6.2]{ref:bertsimas2005optimal} because the set $\{x\}$ is a singleton, and hence
it is closed. This establishes equality~\eqref{eq:non-param:refor0}, where
the maximization operator in the right hand side of~\eqref{eq:non-param:refor0} is justified because~$\mc %
U_\rho(\msa, \covsa)$ is compact by Lemma~\ref{lemma:U-compact} and the
objective function is continuous over $\mc U_\rho(\msa, \covsa)$.

It remains to find the optimal solution $(\m\opt, \cov\opt)$ that solves the
maximization problem~\eqref{eq:non-param:refor0}. If $x = \msa$ then the
optimal value of problem~\eqref{eq:non-param:refor0} is trivially 0. It
suffices to consider the case when $x \neq \msa$. Define $\overline\rho \Let %
\rho + d + \log\det \covsa$. Using a reparametrization $\Omega \leftarrow %
\cov^{-1}$, the maximizer $(\m\opt, \cov\opt)$ also solves
\begin{equation}  \label{eq:non-param:1}
\begin{array}{cl}
\min & (\m - x)^\top \Omega (\m - x) \\
\st & \m \in \R^d,~\Omega \in \PD^d \\
& (\m - \msa)^\top \Omega (\m - \msa) + \Tr{\covsa \Omega} - \log\det \Omega
\leq \overline\rho.%
\end{array}%
\end{equation}
This optimization problem with decision variables $(\m, \Omega)$ is still a
non-convex optimization problem because of the multiplication terms between $%
\m$ and $\Omega$. However, it can be re-expressed as
\begin{equation*}
\begin{array}{cl}
\min & \min ~ (\m - x)^\top \Omega (\m - x) \\
& \st ~~~ \m \in \R^d,~(\m - \msa)^\top \Omega (\m - \msa) \leq \overline\rho - %
\Tr{\covsa \Omega} + \log\det \Omega \\
\st & \Omega \in \PD^d,~\Tr{\covsa \Omega} - \log\det \Omega \leq \overline\rho,%
\end{array}%
\end{equation*}
where we note that the constraint $\Tr{\covsa \Omega} - \log\det \Omega \leq
\overline\rho$ is redundant, but it is added to ensure that the inner problem
over $\mu$ is feasible for any feasible value of $\Omega$ in the outer
problem. Applying Lemma~\ref{lemma:g-refor} to solve the inner problem over $%
\m$ for any given $\Omega \in \PD^d$, problem~\eqref{eq:non-param:1} is
equivalent to
\begin{equation*}
\begin{array}{cl}
\min & \Max{\lambda \ge 0}~ -\lambda (\overline\rho - \Tr{\covsa\Omega} +
\log\det \Omega) + \frac{\lambda}{1+\lambda} (x - \msa)^\top \Omega (x - \msa%
) \\
\st & \Omega \in \PD^d,~\Tr{\covsa \Omega} - \log\det \Omega \leq \overline\rho.%
\end{array}
\end{equation*}
For any $\covsa \in \PD^d$ and $\overline\rho =\rho + d + \log\det \covsa \in \R$%
, the feasible set $\{ \Omega \in \PD^d: \Tr{\covsa \Omega} - \log\det
\Omega \leq \overline\rho\}$ is compact\footnote{
Compactness follows from a reasoning similar to the proof of Lemma~\ref{lemma:U-compact}, thus the details are omitted.} and convex. Moreover, the objective function is convex in $\Omega$ and concave in $\lambda$. Applying
Sion's minimax theorem~\cite{ref:sion1958minimax}, we can interchange the
operators and obtain an equivalent problem
\begin{equation*}
\begin{array}{ccl}
\Max{\lambda \ge 0} & \min & ~ -\lambda (\overline\rho - \Tr{\covsa\Omega} +
\log\det \Omega) + \frac{\lambda}{1+\lambda} (x - \msa)^\top \Omega (x - \msa%
) \\
& \st & \Omega \in \PD^d~,~\Tr{\covsa \Omega} - \log\det \Omega \leq \overline\rho.%
\end{array}%
\end{equation*}
For any $\lambda \ge 0$, we can use a duality argument to reformulate the
inner minimization, and we obtain the equivalent problem
\begin{align*}
& \Max{\lambda \ge 0} \Inf{\Omega \in \PD^d} \Max{\nu \ge 0}~-(\lambda +
\nu) \overline\rho - (\lambda + \nu) \log\det \Omega + (\lambda + \nu)\Tr{\Omega
\covsa} + \frac{\lambda}{1+\lambda} (x - \msa)^\top \Omega(x - \msa) \\
=& \Max{\substack{\lambda \ge 0 \\ \nu \ge 0 }} \Inf{\Omega \in \PD^d}
-(\lambda + \nu) \overline\rho - (\lambda + \nu) \log\det \Omega + (\lambda + \nu)%
\Tr{\Omega \covsa} + \frac{\lambda}{1+\lambda} (x - \msa)^\top \Omega(x - %
\msa),
\end{align*}
where the interchange of the infimum operator with the innermost maximum operator is justified thanks to~\citet[Proposition~5.5.4]{ref:bertsekas2009convex}. Using a change of variables $\dualvar \leftarrow
\lambda + \nu$, problem~\eqref{eq:non-param:1} is equivalent to
\begin{equation*}
\Max{\dualvar \ge \lambda \ge 0} ~\left\{ \varphi(\dualvar, \lambda) \Let %
\Inf{\Omega \in \PD^d} -\dualvar \overline\rho - \dualvar \log\det \Omega + %
\dualvar\Tr{\Omega \covsa} + \frac{\lambda}{1+\lambda} (x - \msa)^\top
\Omega(x - \msa) \right\}.
\end{equation*}
If $\dualvar = \lambda = 0$, we have $\varphi(0, 0) = 0$. For any $\lambda
\ge 0$ and $\dualvar \ge \lambda$ such that $\dualvar > 0$, the inner
minimization admits the optimal solution \be \label{eq:Omega-opt} \Omega\opt%
(\lambda, \dualvar) = \Big( \covsa + \frac{\lambda}{\dualvar(1+\lambda)} (x
- \msa)(x - \msa)^\top \Big)^{-1}. \ee
Furthermore, because $\dualvar > 0$, the inner minimization problem has a strictly convex objective function over $\PD^d$, in this case, the minimizer $\Omega\opt(\lambda, \dualvar)$ is unique. By substituting the value of the minimizer $\Omega\opt(\lambda, \dualvar)$, we obtain
\begin{align*}
\varphi(\dualvar, \lambda) &= ~ -\dualvar \rho + \dualvar \log\det \Big(
I_d + \frac{\lambda}{\dualvar(1+\lambda)} \covsa^{-\half}(x - \msa)(x - \msa)^\top \covsa^{-\half} \Big) \\
&= -\dualvar \rho + \dualvar\log \Big( 1 + \frac{\lambda}{\dualvar%
(1+\lambda)} (x - \msa)^\top \covsa^{-1} (x - \msa) \Big)
\end{align*}
where $\covsa^{-\half}$ denotes the inverse of the unique principal square
root of $\covsa$. In the second equality, we have used~\citet[Fact~2.16.3]{ref:bernstein2009matrix} which implies that
\begin{equation*}
\det (I_d + ab^\top) = 1 + b^\top a \qquad \forall (a, b) \in \R^d \times \R%
^d.
\end{equation*}
In the next step, we show that for any $\dualvar \ge \lambda$, the optimal
solution for the variable $\lambda$ is $\lambda\opt(\dualvar) = \dualvar$.
To this end, rewrite the above optimization problem as a two-layer
optimization problem
\begin{equation*}
\Max{\dualvar \ge 0}~ \Max{\substack{\lambda \ge 0 \\ \lambda \le \dualvar}}%
~ \varphi(\dualvar,\lambda) .
\end{equation*}
This claim is trivial if $\dualvar = 0$ because in this case, the only
feasible solution for $\lambda$ is $\lambda\opt(0) = 0$. If $\dualvar > 0$,
the gradient of $\varphi$ in the variable $\lambda$ satisfies
\begin{equation*}
\frac{\partial \varphi}{\partial \lambda} = \frac{\dualvar
(x-\msa)^\top \covsa^{-1} (x-\msa)}{(1+\lambda)(\gamma(1+\lambda) + \lambda(x-\msa)^\top 
\covsa^{-1} (x-\msa))} \geq 0 \quad \forall \lambda \in [0, \dualvar],
\end{equation*}
which implies that at optimality, we have $\lambda\opt(\dualvar) = \dualvar$%
. Thus, we can eliminate the variable $\lambda$ and obtain the equivalent
univariate optimization problem
\begin{equation*}
\Max{\dualvar \ge 0 } ~ -\dualvar \rho + \dualvar\log \Big( 1 + \frac{1}{1+\dualvar} (x - \msa)^\top \covsa^{-1} (x - \msa) \Big).
\end{equation*}
Converting this problem into a minimization problem gives the formulation~%
\eqref{eq:non-param:refor2}. By studying the objective function of~%
\eqref{eq:non-param:refor2} and its gradient and Hessian\footnote{%
The closed form expressions can be found in Section~\ref{sect:gradient}.},
one can verify that this objective function is strictly convex and it tends
to infinity as $\dualvar$ goes to infinity. This implies that the minimizer $%
\dualvar\opt$ of~\eqref{eq:non-param:refor2} exists and is unique. Let $%
\dualvar\opt$ be the minimizer of~\eqref{eq:non-param:refor2}, one can
reconstruct $\cov\opt$ from~\eqref{eq:Omega-opt} and $\mu\opt$ from Lemma~\ref{lemma:g-refor}, which gives the expression~\eqref{eq:non-param:refor1}. This
observation completes the proof.
\end{proof}


\section{Proof of Section~\protect\ref{sect:param}}

\begin{proof}[Proof of Theorem~\protect\ref{thm:Gauss}]
Evaluating the optimistic score under the Gaussian assumption is
equivalent to solving a non-convex minimization problem \be \label{eq:likelihood} \min \left\{ (\m - x)^\top \cov^{-1} (\m - x) + \log\det \cov%
: (\m, \cov) \in \U_\rho(\msa, \covsa) \right\}, \ee
where the minimization operator is justified by the compactness of the
uncertainty set $\mc U_{\rho}(\msa, \covsa)$ in Lemma~\ref{lemma:U-compact}.
Define $\overline \rho \Let \rho + d + \log\det \covsa$. Using a
reparametrization $\Omega \leftarrow \cov^{-1}$, problem~%
\eqref{eq:likelihood} admits an equivalent formulation
\begin{equation*}
\begin{array}{cl}
\min & \min ~ (\m - x)^\top \Omega (\m - x) - \log\det \Omega \\
& \st ~~~ \m \in \R^d,~(\m - \msa)^\top \Omega (\m - \msa) \leq \overline\rho - %
\Tr{\covsa \Omega} + \log\det \Omega \\
\st & \Omega \in \PD^d,~\Tr{\covsa \Omega} - \log\det \Omega \leq \overline\rho,%
\end{array}%
\end{equation*}
where we emphasize that the constraint~$\Tr{\covsa \Omega} - \log\det \Omega
\leq \overline\rho$ is redundant to ensure the feasibility of the inner problem
over $\mu$ for each admissible $\Omega$. Applying Lemma~\ref{lemma:g-refor}
to solve the inner problem over $\m$ for any given $\Omega \in \PD^d$,
problem~\eqref{eq:likelihood} is equivalent to
\begin{equation*}
\left\{
\begin{array}{cl}
\min & \Max{\lambda \ge 0}~ -\lambda (\overline\rho - \Tr{\covsa\Omega})
-(\lambda+1) \log\det \Omega + \frac{\lambda}{1+\lambda} (x - \msa)^\top
\Omega (x - \msa) \\
\st & \Omega \in \PD^d,~\Tr{\covsa \Omega} - \log\det \Omega \leq \overline\rho.%
\end{array}
\right.
\end{equation*}
Follow a similar steps as in the proof of Theorem~\ref{thm:nonparam}, we
find that problem~\eqref{eq:likelihood} is equivalent to
\begin{equation*}
\Max{\dualvar \ge \lambda \ge 0} ~\left\{ \varphi(\dualvar, \lambda) \Let %
\Inf{\Omega \in \PD^d} -\dualvar \overline\rho - (\dualvar + 1) \log\det \Omega + %
\dualvar\Tr{\Omega \covsa} + \frac{\lambda}{1+\lambda} (x - \msa)^\top
\Omega(x - \msa) \right\}.
\end{equation*}
For any $\lambda \ge 0$ and $\dualvar \ge \lambda$ such that $\dualvar > 0$,
the inner minimization admits the optimal solution \be \label{eq:param:Omega-opt} \Omega\opt(\lambda, \dualvar) = \Big( \frac{\dualvar}{1+\dualvar} \covsa + \frac{\lambda}{(1+\dualvar)(1+\lambda)} (x - \msa)(x - %
\msa)^\top \Big)^{-1}, \ee
By substituting the value of the minimizer $\Omega\opt(\lambda, \dualvar)$,
we obtain
\begin{align*}
\varphi(\dualvar, \lambda) &= ~ (d + \log\det \covsa) - \dualvar\rho - {%
d}(\dualvar + 1) \log\Big( 1 + \frac{1}{\dualvar} \Big) + (1+%
\dualvar) \log \Big(1 + \frac{\lambda (x - \msa)^\top \covsa ^{-1}(x - \msa)%
}{\dualvar (1+\lambda)}\Big),
\end{align*}
If $\dualvar = \lambda = 0$, we have $\varphi(0, 0) = -\infty$ because the
objective value in this case tends to $-\infty$ as $\Omega$ tends to $%
+\infty $. Thus without loss of optimality, we can omit the variable $%
\dualvar = \lambda = 0$ from the outer maximization problem because this set of solution is never optimal. Problem~%
\eqref{eq:likelihood} is hence equivalent to the following two-layer
optimization problem \be \label{eq:param-2layer} \Max{\dualvar > 0}~%
\Max{0
\le \lambda \le \dualvar}~ \varphi(\dualvar, \lambda), \ee
where we emphasize that the feasible set for $\dualvar$ is over the open set $(0, +\infty)$. For any $\dualvar > 0$, the gradient of $\varphi$ in $\lambda$ satisfies
\begin{equation*}
\frac{\partial \varphi}{\partial \lambda} = \frac{\dualvar (1 + \dualvar) (x
- \msa)^\top \covsa^{-1} (x - \msa)}{\dualvar (1 + \lambda) + \lambda (x - %
\msa)^\top \covsa^{-1} (x - \msa)} \ge 0 \qquad \forall \lambda \in [0, %
\dualvar],
\end{equation*}
which implies that the inner maximization problem in~\eqref{eq:param-2layer}
admits the optimal solution $\lambda\opt(\dualvar) = \dualvar$. We thus have
\begin{equation*}
\Max{\dualvar > 0}~ (d + \log\det \covsa) - \dualvar\rho - d(%
\dualvar + 1) \log\Big( 1 + \frac{1}{\dualvar} \Big) + (1+\dualvar) \log
\Big(1 + \frac{ (x - \msa)^\top \covsa ^{-1}(x - \msa)}{ (1+\dualvar)}%
\Big)
\end{equation*}
Dropping the constant term in the objective function and converting the
problem into the minimization form results in problem~\eqref{eq:gauss:refor2}
. By studying the objective function of~\eqref{eq:gauss:refor2} and its
gradient and Hessian\footnote{The closed form expressions can be found in Section~\ref{sect:gradient}.},
one can verify that this objective function is strictly convex and it tends
to infinity as $\dualvar$ goes to infinity. This implies that the minimizer $\dualvar\opt$ of~\eqref{eq:gauss:refor2} exists and is unique. Let $\dualvar\opt$ be the minimizer of~\eqref{eq:gauss:refor2}, one can reconstruct $\cov%
\opt$ from~\eqref{eq:param:Omega-opt} and $\mu\opt$ from Lemma~\ref{lemma:g-refor}, which give expression~\eqref{eq:gauss:refor1}. This finishes the proof.
\end{proof}


\section{Calculations of the Gradients and Hessians}
\label{sect:gradient}

 Throughout this section, we use the shorthand $\alpha =(x-\msa)^{\top }\covsa^{-1}(x-\msa)\geq 0$. Denote momentarily by $%
\varphi _{1}:\R_{+}\rightarrow \R$ the objective function of problem~%
\eqref{eq:non-param:refor2}, that is,
\begin{equation*}
\varphi _{1}(\dualvar)=\dualvar\rho -\dualvar\log \Big( 1+\frac{\alpha }{1+%
\dualvar}\Big) .
\end{equation*}%
The gradient and Hessian of $\varphi _{1}$ are
\begin{align}
\frac{\partial \varphi _{1}}{\partial \dualvar}& =\rho -\log \Big( 1+\frac{\alpha }{1+\dualvar}\Big) +\frac{\dualvar\alpha }{(1+\dualvar)[1+\dualvar+\alpha ]},  \notag \\
\frac{\partial^{2}\varphi_{1}}{\partial \dualvar^{2}}& =\frac{\alpha (2+2%
\dualvar+2\alpha +\alpha \dualvar)}{(1+\dualvar)^{2}(1+\dualvar+\alpha )^{2}}%
\geq 0.  \notag
\end{align}%
Now, denote momentarily by $\varphi _{2}:\R_{++}\rightarrow \R$ the
objective function of problem~\eqref{eq:gauss:refor2}, that is,
\begin{equation*}
\varphi _{2}(\dualvar)=\dualvar\rho +{d}(\dualvar+1)\log \Big( 1+\frac{1}{\dualvar}\Big) -(1+\dualvar)\log \Big( 1+\frac{\alpha }{(1+\dualvar)}\Big) .
\end{equation*}%
The gradient and Hessian of $\varphi _{2}$ are
\begin{align*}
\frac{\partial \varphi _{2}}{\partial \dualvar}& =\rho +d\left[ \log \Big(1+\frac{1}{\dualvar}\Big) -\frac{1}{\dualvar}\right] -\left[ \log \Big( 1+\frac{\alpha }{1+\dualvar}\Big) -\frac{\alpha }{1+\dualvar+\alpha }\right]
, \\
\frac{\partial ^{2}\varphi _{2}}{\partial \dualvar^{2}}& =\frac{d}{\dualvar%
^{2}(1+\dualvar)}+\frac{\alpha ^{2}}{(1+\dualvar+\alpha )^{2}(1+\dualvar)}
\geq 0.
\end{align*}

\section*{Acknowledgements} We gratefully acknowledge support from the following NSF grants~1915967,~1820942,~1838676 as well as the China Merchants Bank.

\bibliographystyle{icml2020}
\bibliography{bibliography, rob_bayesian}
\end{document}